\documentclass{article}

\usepackage[final]{neurips_2024}

\usepackage[utf8]{inputenc} 
\usepackage[T1]{fontenc}    
\usepackage{hyperref}       
\usepackage{url}            
\usepackage{booktabs}       
\usepackage{amsmath,amssymb,amsthm}
\usepackage{amsfonts}       
\usepackage{nicefrac}       
\usepackage{microtype}      
\usepackage{xcolor}         
\usepackage{algorithm}
\usepackage[noend]{algpseudocode}
\usepackage{subcaption}
\usepackage{cancel}
\usepackage{todonotes}
\setuptodonotes{inline}
\usepackage{pifont}
\usepackage{comment}
\usepackage{multibib}
\usepackage{wrapfig}
\usepackage{tikz}
\usetikzlibrary{bayesnet}
\newcites{app}{Supplementary References}

\newcommand{\cmark}{\ding{51}}
\newcommand{\xmark}{\ding{55}}

\usepackage{importance}
\newtheorem{definition}{Definition}
\newtheorem{proposition}{Proposition}
\newtheorem{theorem}{Theorem}

\newtheorem{corollary}{Corollary}[proposition]

\title{Divide-and-Conquer Predictive Coding: a Structured Bayesian Inference Algorithm}

\author{
    \textbf{Eli Sennesh$^{1}$, Hao Wu$^{2}$, Tommaso Salvatori$^{2,3}$}
    \\ $^1$Department of Psychology, Vanderbilt University, Nashville, TN, USA
    \\ $^2$VERSES AI Research Lab, Los Angeles, USA
    \\ $^3$Vienna University of Technology, Vienna, Austria
    \\ \texttt{eli.sennesh@vanderbilt.edu}, \  \texttt{wuhaomxhy@gmail.com}, \  \texttt{tommaso.salvatori@verses.ai}
}

\begin{document}

\maketitle

\begin{abstract}
  Unexpected stimuli induce ``error'' or ``surprise'' signals in the brain. The theory of predictive coding promises to explain these observations in terms of Bayesian inference by suggesting that the cortex implements variational inference in a probabilistic graphical model. However, when applied to machine learning tasks, this family of algorithms has yet to perform on par with other variational approaches in high-dimensional, structured inference problems. To address this, we introduce a novel predictive coding algorithm for structured generative models, that we call divide-and-conquer predictive coding (DCPC); it differs from other formulations of predictive coding, as it respects the correlation structure of the generative model and provably performs maximum-likelihood updates of model parameters, all without sacrificing biological plausibility. Empirically, DCPC achieves better numerical performance than competing algorithms and provides accurate inference in a number of problems not previously addressed with predictive coding. We provide an \href{https://github.com/esennesh/dcpc_paper}{open implementation} of DCPC in Pyro on Github.
\end{abstract}

\vspace{-1em}
\section{Introduction}
\label{sec:introduction}
\vspace{-1em}

In recent decades, the fields of cognitive science, machine learning, and theoretical neuroscience have borne witness to a flowering of successes in modeling intelligent behavior via statistical learning. Each of these fields has taken a different approach: cognitive science has studied probabilistic \emph{inverse inference}~\citep{Chater2006,Pouget2013,Lake2017} in models of each task and environment, machine learning has employed the backpropagation of errors~\citep{Rumelhart1986,Lecun2015,Schmidhuber2015}, and neuroscience has hypothesized that \emph{predictive coding} (PC)~\citep{Srinivasan1982,Rao1999,Friston2005,Bastos2012,Spratling2017,Hutchinson2019,Millidge2021} may explain neural activity in perceptual tasks. These approaches share in common a commitment to ``deep'' models, in which task processing emerges from the composition of elementary units.

In machine learning, PC-based algorithms have recently gained popularity for their theoretical potential to provide a more biologically plausible alternative to backpropagation for training neural networks~\citep{Salvatori2023,Song2024}. However, PC does not perform comparably in these tasks to backpropagation due to limitations in current formulations. First, predictive coding for gradient calculation typically models every node in the computation graph with a Gaussian, and hence fails to express many common generative models. Recent work on PC has addressed this by allowing approximating non-Gaussian energy functions with samples~\citep{Pinchetti2022}. Second, the Laplace approximation to the posterior infers only a maximum-a-posteriori (MAP) estimate and Gaussian covariance for each latent variable, keeping PC from capturing multimodal or correlated distributions. Third, this loose approximation to the posterior distribution results in inaccurate, high-variance updates to the parameters of the generative model.

In this work we propose a new algorithm, \emph{divide-and-conquer predictive coding} (DCPC), for approximating structured target distributions with populations of Monte Carlo samples. DCPC goes beyond Gaussian assumptions, and decomposes the problem of sampling from structured targets into local coordinate updates to individual random variables. These local updates are informed by unadjusted Langevin proposals parameterized in terms of biologically plausible prediction errors. Nesting the local updates within divide-and-conquer Sequential Monte Carlo~\citep{Lindsten2017,Kuntz2024} ensures that DCPC can target any statically structured graphical model, while Theorem~\ref{thm:ppc_local_likelihood} provides a locally factorized way to learn model parameters by maximum marginal likelihood.

DCPC also provides a computational perspective on the canonical cortical microcircuit~\citep{Bastos2012,Bastos2020,Campagnola2022} hypothesis in neuroscience. Experiments have suggested that deep laminar layers in the cortical microcircuit represent sensory imagery, while superficial laminar represent raw stimulus information~\citep{Bergmann2024}; experiments in a predictive coding paradigm specifically suggested that the deep layers represent ``predictions'' while the shallow layers represent ``prediction errors''.
This circuitry could provide the brain with its fast, scalable, generic Bayesian inference capabilities. Figure~\ref{fig:dcpc-structure} compares the computational structure of DCPC with that of previous PC models. The following sections detail the contributions of this  work:
\begin{itemize}
    \item Section~\ref{sec:algorithm} defines the divide-and-conquer predictive coding algorithm and shows how to use it as a variational inference algorithm;
    \item Section~\ref{sec:bioplausibility} examines under what assumptions the cortex could plausibly implement DCPC, proving two theorems that contribute to biological plausibility;
    \item Section~\ref{sec:experiments} demonstrates DCPC experimentally in head-to-head comparisons against recent generative models and inference algorithms from the predictive coding literature.
\end{itemize}
Section~\ref{sec:background} will review the background for Section~\ref{sec:algorithm}'s algorithm: the problem predictive coding aims to solve and a line of recent work adressing that problem from which this paper draws.

\begin{figure}
    \begin{subfigure}{0.45\textwidth}
        \centering
        \includegraphics[width=\columnwidth]{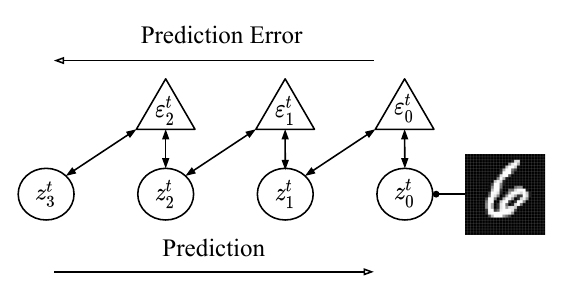}
    \end{subfigure}
    \begin{subfigure}{0.45\textwidth}
        \centering
        \includegraphics[width=\columnwidth]{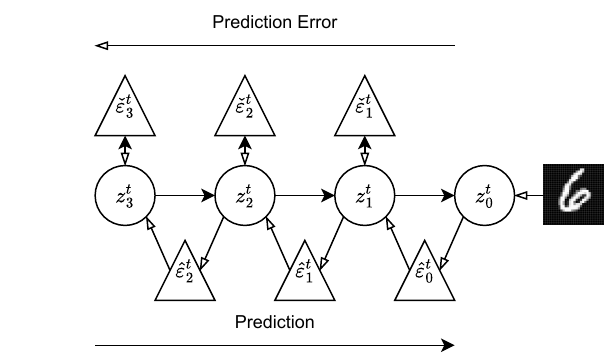}
    \end{subfigure}
    \caption{\textbf{Left}: Classical PC learns a mean-field approximate posterior with prediction error layers. \textbf{Right}: Divide-and-conquer PC approximates the joint posterior with bottom-up and recurrent errors. Where classical predictive coding has layers communicate through shared error units, divide-and-conquer predictive coding separates recurrent from ``bottom-up'' error pathways to target complete conditional distributions rather than posterior marginal distributions.}
    \label{fig:dcpc-structure}
    \vspace{-1.5em}
\end{figure}

\vspace{-1em}
\section{Background}
\label{sec:background}
\vspace{-1em}


This section reviews the background necessary to construct the divide-and-conquer predictive coding algorithm in Section~\ref{sec:algorithm}. Let us assume we have a directed, acyclic graphical model with a joint density split into observations $x \in \mathbf{x}$ and latents $z \in \mathbf{z}$, parameterized by some $\theta$ at each conditional density
\begin{align}
    \label{eq:factorization}
    \priorm{\theta}{\mathbf{x}, \mathbf{z}} &:= \prod_{x \in \mathbf{x}} \priorc{\theta}{x}{\parents{x}} \prod_{z \in \mathbf{z}} \priorc{\theta}{z}{\parents{z}},
\end{align}
where $\parents{z}$ denotes the parents of the random variable $z \in \mathbf{z}$ and $\children{z}$ denotes its children.

\paragraph{Empirical Bayes}
\emph{Empirical Bayes} consists of jointly estimating, in light of the data, both the parameters $\theta^{*}$ and the Bayesian posterior over the latent variables $\mathbf{z}$, that is:
\begin{align*}
    \theta^{*} &= \arg \max_{\theta} \priorm{\theta}{\mathbf{x}} = \arg \max_{\theta} \expectint{\priorm{\theta}{\mathbf{x}, \mathbf{z}}\: d\mathbf{z}}{\mathbf{z} \in \mathcal{Z}}{}, &
    \priorc{\theta^{*}}{\mathbf{z}}{\mathbf{x}} &:= \frac{
        \priorm{\theta^{*}}{\mathbf{x}, \mathbf{z}}
    }{
        \priorm{\theta^{*}}{\mathbf{x}}
    }.
\end{align*}

Typically the marginal and posterior densities have no closed form, so learning and inference algorithms treat the joint distribution as a closed-form \emph{un}normalized density over the latent variables; its integral then gives the normalizing constant for approximation
\begin{align*}
    \targetm{\theta}{\mathbf{z}} &:= \priorm{\theta}{\mathbf{x}, \mathbf{z}}, &
    Z_{\theta} &:= \expectint{\targetm{\theta}{\mathbf{z}}\: d\mathbf{z}}{\mathbf{z} \in \mathcal{Z}}{} = \priorm{\theta}{\mathbf{x}}, &
    \posteriorm{\theta}{\mathbf{z}} &:= \frac{\targetm{\theta}{\mathbf{z}}}{Z_{\theta}}.
\end{align*}
\citet{Neal1998} reduced empirical Bayes to minimization of the \emph{variational free energy}:
\begin{align}
\label{eq:vfe}
    \mathcal{F}(\theta, q) &:= \expect{\mathbf{z}\sim q(\mathbf{z})}{-\log \frac{\targetm{\theta}{\mathbf{z}}}{q(\mathbf{z})}} \geq -\log \normalizer{}{\theta}.
\end{align}
The ratio of densities in Equation~\ref{eq:vfe} is an example of a \emph{weight} used to approximate a distribution known only up to its normalizing constant. The \emph{proposal} distribution $q(\mathbf{z})$ admits tractable sampling, while the unnormalized \emph{target} density $\targetm{\theta}{\mathbf{z}}$ admits tractable, pointwise density evaluation.
\paragraph{Predictive Coding}
Computational neuroscientists now often hypothesize that \emph{predictive coding} (PC) can optimize the above family of objective functionals in a local, neuronally plausible way \citep{Millidge2021,Millidge2023}. More in detail, it is possible to define this class of algorithms as follows: 
\begin{definition}[Predictive Coding Algorithm]
\label{def:pc_alg}
Consider approximate inference in a model $\priorm{\theta}{\mathbf{x}, \mathbf{z}}$ using an algorithm $\mathcal{A}$. \citet{Salvatori2023} calls $\mathcal{A}$ a \emph{predictive coding algorithm} if and only if:
\begin{enumerate}
    \item It maximizes the model evidence $\log \priorm{\theta}{\mathbf{x}}$ by minimizing a variational free energy;
    \item The proposal $\proposalm{}{\mathbf{z}} = \prod_{z \in \mathbf{z}} \proposalm{}{z}$ factorizes via a mean-field approximation; and
    \item Each proposal factor is a Laplace approximation (i.e.~$\proposalm{\mu}{z} := \mathcal{N}(\mu, \Sigma(\mu))$).
\end{enumerate}
\end{definition}

\paragraph{Particle Algorithms} In contrast to predictive coding, particle algorithms approach empirical Bayes problems by setting the proposal to a collection of weighted particles $(w^{k}, \mathbf{z}^{k})$ drawn from a sampling algorithm meeting certain conditions (see Definition~\ref{def:spw_density} in Appendix~\ref{app:formal_theorems}). Any proposal meeting these conditions (see Proposition~\ref{prop:elbo_bound} in Appendix~\ref{app:formal_theorems} and \citet{Naesseth2015,Stites2021}) defines a free energy functional, analogous to Equation~\ref{eq:vfe} in upper-bounding the model surprisal:
\begin{align*}
    \mathcal{F}(\theta, q) := \expect{w, \mathbf{z} \sim q(w, \mathbf{z})}{-\log w}
    &\implies
    \mathcal{F}(\theta, q) \geq -\log \normalizer{}{\theta}.
\end{align*}
This paper builds on the particle gradient descent (PGD) algorithm of \citet{Kuntz2023}, that works as follows: At each iteration $t$, PGD diffuses the particle cloud $\proposalm{K}{\mathbf{z}} = \frac{1}{K}\sum_{k=1}^{K} \dirac{\mathbf{z}^{k}}{\mathbf{z}}$ across the target log-density with a learning rate $\eta$ and independent Gaussian noise; it then updates the parameters $\theta$ by ascending the gradient of the log-likelihood, estimated by averaging over the particles. The update rules are then the following:
\begin{align}
    \label{eq:pgd_z}
    \mathbf{z}^{t+1, k} &:= \mathbf{z}^{t, k} + \eta \nabla_{\mathbf{z}} \log \targetm{\theta^{t}}{\mathbf{z}^{t, k}} + \sqrt{2\eta} \xi^k, \\
    \label{eq:pgd_theta}
    \theta^{t+1} &:= \theta^{t} + \eta \left(\frac{1}{K} \sum_{k=1}^{K} \nabla_{\theta} \log \targetm{\theta^{t}}{\mathbf{z}^{t+1, k}} \right).
\end{align}
The above equations target the joint density of an entire graphical model\footnote{\citet{Kuntz2023} also interpreted Equation~\ref{eq:pgd_z} as an update step along the Wasserstein gradient in the space of probability measures. Appendix~\ref{app:discrete_wasserstein} extends this perspective to predictive coding of discrete random variables.}. When the prior $\priorm{\theta}{\mathbf{z}}$ factorizes into many separate conditional densities, achieving high inference performance often requires factorizing the inference network or algorithm into conditionals as well~\citep{Webb2018}. Estimating the gradient of the entire log-joint, as in PGD and amortized inference~\citep{Dasgupta2020,Peters2024}, also requires nonlocal backpropagation. To provide a generic inference algorithm for high-dimensional, structured models using only local computations, Section~\ref{sec:algorithm} will apply Equation~\ref{eq:pgd_z} to sample individual random variables in a joint density, combine the coordinate updates via sequential Monte Carlo, and locally estimate gradients for model parameters via Equation~\ref{eq:pgd_theta}.

\vspace{-1em}
\section{Divide-and-Conquer Predictive Coding}
\label{sec:algorithm}
\vspace{-1em}

\begin{table}[t]
    \centering
    \begin{tabular}{c|c|c|c|c}
        \hline
         & PC & LPC & MCPC & DCPC (ours) \\
        \hline
        Generative density & Gaussian & Differentiable & Gaussian & Differentiable \\
        Inference approximation & Laplace & Gaussian & Empirical & Empirical \\
        Posterior conditional structure & \xmark & \xmark & \xmark & \cmark \\
        \hline
    \end{tabular}
    \caption{Comparison of divide-and-conquer predictive coding (DCPC) against other predictive coding algorithms. DCPC provides the greatest flexibility: arbitrary differentiable generative models, an empirical approximation to the posterior, and sampling according to the target's conditional structure.}
    \label{tab:ppc_vs_pcs}
    \vspace{-2em}
\end{table}

The previous section provided a mathematical toolbox for constructing Monte Carlo algorithms based on gradient updates and a working definition of predictive coding. This section will combine those tools to generalize the above notion of predictive coding, yielding the novel \emph{divide-and-conquer predictive coding} (DCPC) algorithm. Given a causal graphical model, DCPC will approximate the posterior with a population $q(\mathbf{z})$ of $K$ samples, while also learning $\theta$ explaining the data. This will require deriving  local coordinate updates and then parameterizing them in terms of prediction errors.

Let us assume we again have a causal graphical model $\priorm{\theta}{\mathbf{x}, \mathbf{z}}$ locally parameterized by $\theta$ and factorized (as in Equation~\ref{eq:factorization}) into conditional densities for each $x \in \mathbf{x}$ and $z \in \mathbf{z}$. DCPC then requires two hyperparameters: a learning rate $\eta \in \mathbb{R}^{+}$, and particle count $K \in \mathbb{N}^{+}$, and is initialized  (at $t=0$) via a population of predictions by ancestor sampling defined as $\mathbf{z}^{0} \sim \prod_{z \in \mathbf{z}} \priorc{\theta}{z^{0}}{\parents{z^{0}}}$.


DCPC aims to minimize the variational free energy (Equation~\ref{eq:vfe}).
The optimal proposal $q_{*}$ for each random variable would equal, if it had closed form, the \emph{complete conditional} density for that variable, containing all information from other random variables
\begin{align}
    \label{eq:complete_conditional}
    \proposalc{*}{\mathbf{z}^{t}}{\mathbf{z}^{t-1}} \propto \targetc{\theta}{z}{\mathbf{z}_{\setminus z}} &= \priorc{\theta}{z}{\parents{z}} \prod_{v \in \children{z}} \priorc{\theta}{v}{\parents{v}}.
\end{align}

    
We observe that the prediction errors $\varepsilon_{z}$ in classical predictive coding, usually defined as the precision weighted difference between predicted and actual value of a variable, can be seen as the \emph{score function} of a Gaussian, where the score is the gradient with respect to the parameter $z$ of the log-likelihood:
\begin{align*}
    \varepsilon_{z} &:= \nabla_{z} \log \mathcal{N}(z, \tau) = \tau \left(x - z\right);
\end{align*}
When given the ground-truth parameter $z$, the \emph{expected} score function $\expect{x \sim \priorc{}{x}{z}}{\nabla_{z} \log \priorc{}{x}{z}}$ under the likelihood becomes zero, making score functions a good candidate for implementing predictive coding. We therefore define $\varepsilon_{z}$ in DCPC as the complete conditional's score function
\begin{align}
    \label{eq:pe}
    \varepsilon_{z} &:= \nabla_{z} \log \targetc{\theta}{z}{\mathbf{z}_{\setminus z}} = \nabla_{z} \log \priorc{\theta}{z}{\parents{z}} + \sum_{v \in \children{z}} \nabla_{z} \log \priorc{\theta}{v}{\parents{v}}.
\end{align}
This gradient consists of a sum of local prediction-error terms: one for the local ``prior'' on $z$ and one for each local ``likelihood'' of a child variable. By defining the prediction error as a sum of local score functions, we write Equation~\ref{eq:pgd_z} in terms of $\varepsilon_{z}$ (Equation~\ref{eq:pe}) and the preconditioner of Definition~\ref{def:pc_fim}:
\begin{align*}
    \proposalc{\eta}{z^{t}}{\varepsilon^{t}_{z}, z^{t-1}} &:= \mathcal{N}\left(z^{t-1} + \eta \Hat{\Sigma}_{\mathcal{I}} \varepsilon^{t}_{z} , 2\eta \Hat{\Sigma}_{\mathcal{I}}\right).
\end{align*}
The resulting proposal now targets the complete conditional density (Equation~\ref{eq:complete_conditional}), simultaneously meeting the informal requirement of Definition~\ref{def:pc_alg} for purely local proposal computations while also ``dividing and conquering'' the sampling problem into lower-dimensional coordinate updates.

Since the proposal from which we can sample by predictive coding is not the optimal coordinate update, we importance weight for the true complete conditional distribution that is optimal
\begin{align}
\label{eq:coord_weight}
    z^{t} &\sim \proposalc{\eta}{z^{t}}{z^{t-1}, \varepsilon^{t}_{z}} &
    u^{t}_{z} &= \frac{
        \targetc{\theta^{t-1}}{z^{t}}{\mathbf{z}_{\setminus z}}
    }{
        \proposalc{\eta}{z^{t}}{z^{t-1}, \varepsilon^{t}_{z}}
    };
\end{align}
resampling with respect to these weights corrects for discretization error, yields particles distributed according to the true complete conditional, and estimates the complete conditional's normalizer
\begin{align*}
    \textsc{RESAMPLE}\left(z^{t}, u^{t}_{z} \right) &\sim \posteriorc{\theta^{t-1}}{z^{t}}{\mathbf{z}_{\setminus z}}, &
    \normalizerHat{\theta^{t-1}}{\mathbf{z}_{\setminus z}}^{t} &:= \frac{1}{K} \sum_{k=1}^{K} u^{t, k}_{z}.
\end{align*}
The recursive step of ``Divide and Conquer'' Sequential Monte Carlo~\citep{Lindsten2017,Kuntz2024} exploits the estimates $\normalizerHat{\theta^{t-1}}{\mathbf{z}_{\setminus z}}^{t}$ to weigh the samples for the complete target density
\begin{align}
    \label{eq:target_weight}
    w_{\theta^{t-1}}^{t} &= \frac{
        \priorm{\theta^{t-1}}{\mathbf{x}, \mathbf{z}^{t}}
    }{
        \prod_{z \in \mathbf{z}} \targetc{\theta}{z^{t}}{\mathbf{z}_{\setminus z}}
    }
    \prod_{z \in \mathbf{z}} \normalizerHat{\theta^{t-1}}{\mathbf{z}_{\setminus z}}^{t}.
\end{align}
By Proposition~\ref{prop:elbo_bound}, log-transforming these weights estimates the free energy (Equation~\ref{eq:vfe}):
\begin{align*}
    \mathcal{F}^{t}(\mathbf{z}^{t-1}, \theta^{t-1}) &:= \expect{
        \proposalc{*}{\mathbf{z}^{t}}{\mathbf{z}^{t-1}}
    }{
        - \log \frac{
            \priorm{\theta^{t-1}}{\mathbf{x}, \mathbf{z}^{t}}
        }{
            \proposalc{*}{\mathbf{z}^{t}}{\mathbf{z}^{t-1}}
        }
    } \approx \expect{q}{- \log w_{\theta^{t-1}}^{t}}.
\end{align*}
Theorem~\ref{prop:free_energy_gradient} in Appendix~\ref{app:formal_theorems} shows that the gradient $\nabla_{\theta^{t-1}} \mathcal{F}^{t} = \expect{q}{- \nabla_{\theta^{t-1}} \log \priorm{\theta^{t-1}}{\mathbf{x}, \mathbf{z}^{t}}}$ of the above estimator equals the expected gradient of the log-joint distribution.
Descending this gradient $\theta^{t} := \theta^{t-1} - \eta \nabla_{\theta^{t-1}} \mathcal{F}^{t}$ enables DCPC to learn model parameters $\theta$.

\begin{algorithm}[!tb]
    \caption{Divide-and-Conquer Predictive Coding for empirical Bayes}
    \label{alg:ppc}
\begin{algorithmic}[1]
    \Require learning rate $\eta \in \mathbb{R}^{+}$, particle count $K \in \mathbb{N}$, number of sweeps $S \in \mathbb{N}$
    \Require initial particle vector $\mathbf{z}^{0}$, initial parameters $\theta^{0}$, observations $\mathbf{x} \in \mathcal{X}$
    \For{$t \in [1 \ldots T]$}\label{line:ppc_time_loop}\Comment{Loop through predictive coding steps}
        \For{$s \in [1 \ldots S]$}\label{line:ppc_sweeps_loop}\Comment{Loop through Gibbs sweeps over graphical model}
            \For{$z \in \mathbf{z}$}\label{line:ppc_blocks_loop}\Comment{Loop through latent variables in graphical model}
                \State $\varepsilon_{z} \leftarrow \nabla_{\mathbf{z}} \log \priorc{\theta^{t-1}}{z}{\parents{z}}$ \Comment{Local prediction error}
                \State $\varepsilon_{z} \leftarrow \varepsilon_{z} + \sum_{v \in \children{z}} \nabla_{\mathbf{z}} \log \priorc{\theta^{t-1}}{v}{\parents{v}}$ \Comment{Children's prediction errors}
                \State $\Hat{\Sigma}_{\mathcal{I}} \leftarrow \frac{\Hat{\mathcal{I}}_{K}(\varepsilon^{1:K}_{z})^{-1}}{\frac{1}{d} \mathrm{Tr}[\Hat{\mathcal{I}}_{K}(\varepsilon^{1:K}_{z})^{-1}]}$ \Comment{Estimate precision of prediction errors}
                \State $z^{t} \sim \proposalc{\eta}{z^{t}}{\varepsilon_{z}, z^{t-1}}$ \Comment{Sample coordinate update}
                \State $u^{t}_{z} \leftarrow \frac{\targetc{\theta^{t-1}}{z^{t}}{\mathbf{z}_{\setminus z}}}{\proposalc{\eta}{z^{t}}{\varepsilon_{z}, z^{t-1}}}$ \Comment{Correct coordinate update by weighing}
                \State $z^{t} \leftarrow \textsc{RESAMPLE}\left(z^{t}, u^{t}_{z} \right)$\Comment{Resample from true coordinate update}
                \State $\normalizerHat{\theta^{t-1}}{\mathbf{z}_{\setminus z}}^{t} \leftarrow \frac{1}{K} \sum_{k=1}^{K} u^{t, k}_{z}$\Comment{Estimate coordinate update's normalizer}
            \EndFor
        \EndFor
        \State $\mathcal{F}^{t} \leftarrow -\frac{1}{K} \sum_{k=1}^{K} \log \left(\frac{
            \priorm{\theta^{t-1}}{\mathbf{x}, \mathbf{z}^{t,k}}
        }{
            \prod_{z \in \mathbf{z}} \targetc{\theta^{t-1}}{z^{t,k}}{\mathbf{z}^{t,k}_{\setminus z}}
        }
        \prod_{z \in \mathbf{z}} \normalizerHat{\theta^{t-1}}{\mathbf{z}_{\setminus z}}^{t}\right)$ \Comment{Update free energy}
        \State $\theta^{t} \leftarrow \theta^{t-1} + \eta\frac{1}{K}\sum_{k=1}^{K}
            \nabla_{\theta^{t-1}} \log \priorm{\theta^{t-1}}{\mathbf{x}, \mathbf{z}^{t,k}} $\Comment{Update parameters}
    \EndFor
    \State \Return $\mathbf{z}^{T}$, $\theta^{T}, \mathcal{F}^{T}$\Comment{Output: updated particles, parameters, free energy}
\end{algorithmic}
\end{algorithm}

The above steps describe a single pass of divide-and-conquer predictive coding over a causal graphical model. Algorithm~\ref{alg:ppc} shows the complete algorithm, consisting of nested iterations over latent variables $z \in \mathbf{z}$ (inner loop) and iterations $t \in T$ (outer loop). DCPC satisfies criteria (1) and (2) of Definition~\ref{def:pc_alg}, and relaxes criterion (3) to allow gradient-based proposals beyond the Laplace assumption. As with \citet{Pinchetti2022} and \citet{Oliviers2024}, relaxing the Laplace assumption enables much greater flexibility in approximating the model's true posterior distribution.


\vspace{-1em}
\section{Biological plausibility}
\label{sec:bioplausibility}
\vspace{-1em}

Different works in the literature consider different criteria for biological plausibility. This paper follows the non-spiking predictive coding literature and considers an algorithm biologically plausible if it performs only spatially local computations in a probabilistic graphical model~\citep{whittington2017approximation}, without requiring a global control of computation. However, while in the standard literature locality is either directly defined in the objective function~\citep{Rao1999}, or derived from a mean-field approximation to the joint density~\citep{Friston2005}, showing that the updates of the parameters of DCPC require only local information is not as trivial. To this end, in this section we first formally show that  DCPC achieves decentralized inference of latent variables $\mathbf{z}$ (Theorem~\ref{thm:ppc_complete_conditionals}), and then that also the parameters $\theta$ are updated via local information (Theorem~\ref{thm:ppc_local_likelihood}). 


Gibbs sampling provides the most widely-used algorithm for sampling from a high-dimensional probability distribution by local signaling. It consists of successively sampling coordinate updates to individual nodes in the graphical model by targeting their complete conditional densities $\posteriorc{\theta}{z}{\mathbf{x}, \mathbf{z}_{\setminus z}}$. Theorem~\ref{thm:ppc_complete_conditionals} demonstrates that DCPC's coordinate updates approximate Gibbs sampling.
\begin{theorem}[DCPC coordinate updates sample from the true complete conditionals]
\label{thm:ppc_complete_conditionals}
Each DCPC coordinate update (Equation~\ref{eq:coord_weight}) for a latent $z \in \mathbf{z}$ samples from $z$'s complete conditional (the normalization of Equation~\ref{eq:complete_conditional}). Formally, for every measurable $h: \mathcal{Z} \rightarrow \mathbb{R}$, resampled expectations with respect to the DCPC coordinate update equal those with respect to the complete conditional
\begin{align*}
    \expect{z \sim \proposalc{\eta}{z}{z^{t-1}, \varepsilon^{t}_{z}}}{
        \expect{u \sim \delta(u), z' \sim \textsc{RESAMPLE}\left(z, u_{z} \right)}{
            h(z)
        }
    } &= \expectint{\posteriorc{\theta}{z}{\mathbf{z}_{\setminus z}}\: dz}{z \in \mathcal{Z}}{h(z)}.
\end{align*}
\end{theorem}
\begin{proof}
See Corollary~\ref{cor:ppc_complete_conditionals} in Appendix~\ref{app:formal_theorems}.
\vspace{-1.25em}
\end{proof}
We follow the canonical cortical microcircuit hypothesis of predictive coding~\citep{Bastos2012,Gillon2023} or predictive routing~\citep{Bastos2020}. Consider a cortical column representing $z \in \mathbf{z}$; the $\theta$, $\alpha/\beta$, and $\gamma$ frequency bands of neuronal oscillations~\citep{Buzsaki2004} could synchronize parallelizations (known to exist for simple Gibbs sampling in a causal graphical model~\citep{Gonzalez2011}) of the loops in Algorithm~\ref{alg:ppc}. From the innermost to the outermost and following the neurophysiological findings of \citet{Bastos2015,Fries2015}, $\gamma$-band oscillations could synchronize the bottom-up conveyance of prediction errors (lines 4-6) from L2/3 of lower cortical columns to L4 of higher columns, $\beta$-band oscillations could synchronize the top-down conveyance of fresh predictions (implied in passing from $s$ to $s+1$ in the loop of lines 2-9) from L5/6 of higher columns to L1+L6 of lower columns, and $\theta$-band oscillations could synchronize complete attention-directed sampling of stimulus representations (lines 1-11).
Figure~\ref{fig:laminar_circuitry} in Appendix~\ref{app:experiments} visualizes these hypotheses for how neuronal areas and connections could implement DCPC.


Biological neurons often spike to represent \emph{changes} in their membrane voltage~\citep{Mainen1995,Lundstrom2008,Forkosh2022}, and some have even been tested and found to signal the temporal derivative of the logarithm of an underlying signal~\citep{Adler2018,Borba2021}. Theorists have also proposed models~\citep{Chavlis2021,Moldwin2021} under which single neurons could calculate gradients internally. In short, if neurons can represent probability densities, as many theoretical proposals and experiments suggest they can, then they can likely also calculate the prediction errors used in DCPC. Theorem~\ref{thm:ppc_local_likelihood} will demonstrate that given the ``factorization'' above, DCPC's model learning requires only local prediction errors.


\begin{theorem}[DCPC parameter learning requires only local gradients in a factorized generative model]
\label{thm:ppc_local_likelihood}
Consider a graphical model factorized according to Equation~\ref{eq:factorization}, with the additional assumption that the model parameters $\theta \in \Theta = \prod_{x\in \mathbf{x}} \Theta_{x} \times \prod_{z\in \mathbf{z}} \Theta_{z}$ factorize disjointly. Then the gradient $\nabla_{\theta} \mathcal{F}(\theta, q)$ of DCPC's free energy similarly factorizes into a sum of local particle averages
\begin{align}
    \vspace{-1em}
    \label{eq:vfe_grad_estimator}
    \nabla_{\theta} \mathcal{F} &= \expect{q}{- \nabla_{\theta} \log \priorm{\theta}{\mathbf{x}, \mathbf{z}}} \approx -\sum_{v \in (\mathbf{x}, \mathbf{z})} \frac{1}{K} \sum_{k=1}^{K}\nabla_{\theta_{v}} \log \priorc{\theta_{v}}{v^{k}}{\parents{v}^{k}}.
    \vspace{-1em}
\end{align}
\end{theorem}
\begin{proof}
See Proposition~\ref{prop:ppc_local_likelihood} in Appendix~\ref{app:formal_theorems}. \vspace{-1em}
\end{proof}
Our practical implementation of DCPC, evaluated in the experiments above, takes advantage of Theorem~\ref{thm:ppc_local_likelihood} to save memory by detaching samples from the automatic differentiation graph in the forward ancestor-sampling pass through the generative model.

Finally, DCPC passes from local coordinate updates to the joint target density via an importance resampling operation, requiring that implementations synchronously transmit numerical densities or log-densities for the freshly proposed particle population. While phase-locking to a cortical oscillation may make this biologically feasible, resampling then requires normalizing the weights. Thankfully, divisive normalization appears ubiquitously throughout the brain~\citep{Carandini2012}, as well as just the type of ``winner-take-all'' circuit that implements a softmax function (e.g.~for normalizing and resampling importance weights) being ubiquitous in crosstalk between superficial and deep layers of the cortical column~\citep{Liu1999,Douglas2004}.


\vspace{-1em}
\section{Experiments}
\label{sec:experiments}
\vspace{-1em}
Divide-and-conquer predictive coding is not the first predictive coding algorithm to incorporate sampling into the inference process, and certainly not the first variational inference algorithm for structured graphical models. This section therefore evaluates DCPC's performance against both models from the predictive coding literature and against a standard deep generative model. Each experiment holds the generative model, dataset, and hyperparameters constant except where noted. 

We have implemented DCPC as a variational proposal or ``guide'' program in the deep probabilistic programming language Pyro~\citep{Bingham2019}; doing so enables us to compute free energy and prediction errors efficiently in graphical models involving neural networks. Since the experiments below involve minibatched subsampling of observations $\mathbf{x} \sim \mathcal{B}$ from a dataset $\mathcal{D} \sim p(\mathcal{D})$ of unknown distribution, we replace Equation~\ref{eq:vfe_grad_estimator} with a subsampled form (see \citet{Welling2011} for derivation) of the variational Sequential Monte Carlo gradient estimator~\citep{Naesseth2018}
\begin{align}
    \label{eq:vfe_subsampled_grad}
    \nabla_{\theta} \mathcal{F} &\approx |\mathcal{D}| \expect{\mathcal{B} \sim p(\mathcal{D})}{
        \frac{1}{|\mathcal{B}|} \sum_{\mathbf{x}^{b} \in \mathcal{B}} \expect{(\mathbf{z}, w)^{1:K} \sim q}{
           \log \left(\frac{1}{K} \sum_{k=1}^{K} w^{k} \right)
           \mid
           \mathbf{x}^{b}
        }
    }.
\end{align}
We optimized the free energy in all experiments using Adam~\citep{kingma2014adam}, making sure to call \texttt{detach()} after every Pyro \texttt{sample()} operation to implement the purely local gradient calculations of Theorem~\ref{thm:ppc_local_likelihood} and Equation~\ref{eq:vfe_subsampled_grad}. The first experiment below considers a hierarchical Gaussian model on three simple datasets. The model consists of two latent codes above an observation.



\paragraph{Deep latent Gaussian models with predictive coding}
\citet{Oliviers2024} brought together predictive coding with neural sampling hypotheses in a single model: Monte Carlo predictive coding (MCPC). Their inference algorithm functionally backpropagated the score function of a log-likelihood, applying Langevin proposals to sample latent variables from the posterior joint density along the way. They evaluated MCPC's performance on MNIST with a deep latent Gaussian model~\citep{Rezende2014} (DLGM). \href{https://github.com/gaspardol/MonteCarloPredictiveCoding}{Their model's} conditional densities consisted of nonlinearities followed by linear transformations to parameterize the mean of each Gaussian conditional, with learned covariances. Figure~\ref{fig:graphical-model-dlgm} shows that the DLGM structure already requires DCPC to respect hierarchical dependencies.

\begin{wrapfigure}{l}{0.20\textwidth}
    \begin{tikzpicture}[square/.style={regular polygon,regular polygon sides=4}]
        \node[obs] (x) {$x$};
        \node[latent, above=of x, xshift=-0.0cm] (z2) {$z_2$};
        \node[latent, above=of z2, xshift=-0.0cm] (z1) {$z_1$};
        \node[det, right=of z2](theta){$\theta$};
        \edge{z2} {x};
        \edge{z1} {z2};
        \edge{theta} {z1};
        \edge{theta} {z2};
        \edge{theta} {x};
        \plate[inner sep=.2cm] {pc} {(x)(z2)(z1)} {$N$};
    \end{tikzpicture}
    \caption{Hierarchical graphical model for DLGM's.}
    \label{fig:graphical-model-dlgm}
\end{wrapfigure}
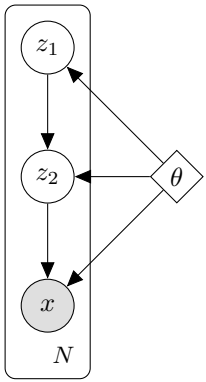

We tested DCPC's performance on elementary reconstruction and generation tasks by using it to train this exact generative model, changing only the likelihood from a discrete Bernoulli to a continuous Bernoulli~\citep{Loaiza2019}. After training we evaluated with a discrete Bernoulli likelihood. Table~\ref{tab:dlgm_results} shows that in terms of both surprise (negative log evidence, with the discrete Bernoulli likelihood) and mean squared reconstruction error, DCPC enjoys better average performance with a lower standard deviation of performance, the latter by an order of magnitude. All experiments used a learning rate $\eta=0.1$ and $K=4$ particles.

\begin{table}[t]
    \centering
    \begin{tabular}{lllll} 
        \hline
        Inference algorithm & Dataset & NLL $\downarrow$ & Mean Squared Error $\downarrow$ \\
        \hline
        MCPC & MNIST & $144.6 \pm 0.7$ & $(8.29 \pm 0.05) \times 10^{-2}$ \\
        DCPC & MNIST & $\mathbf{102.5} \pm \mathbf{0.01}$ & $\mathbf{0.01} \pm \mathbf{7.2 \times 10^{-6}}$ \\
        \hline
        DCPC & EMNIST & $160.8 \pm 0.05$ & $3.3 \times 10^{-6} \pm 3.5 \times 10^{-9}$ \\
        DCPC & Fashion MNIST & $284.1 \pm 0.05$ & $0.03 \pm 2.7 \times 10^{-5}$ \\
        \hline
    \end{tabular}
    \caption{Negative log-likelihood and mean squared error for MCPC against DCPC on held-out images from the MNISTs. Means and standard deviations are taken across five random seeds.}
    \label{tab:dlgm_results}
    \vspace{-2em}
\end{table}

Figure~\ref{fig:ppc_mnists_recons} shows an extension of this experiment to EMNIST~\citep{cohen2017emnist} and Fashion MNIST~\citep{xiao2017fashion} as well as the original MNIST, with ground-truth images in the top row and their reconstructions from DCPC-inferred latent codes below. The ground-truth images come from a 10\% validation split of each data-set, on which DCPC only infers particles $\proposalm{K=4}{\mathbf{z}}$.

The above datasets do not typically challenge a new inference algorithm. The next experiment will thus attempt to learn representations of color images, as in the widely-used variational autoencoder~\citep{kingma2013auto} framework, without an encoder network or amortized inference.

\begin{figure}[b]
    \vspace{-1em}
    \centering
    \begin{subfigure}[t]{0.3\textwidth}
        \includegraphics[width=\columnwidth]{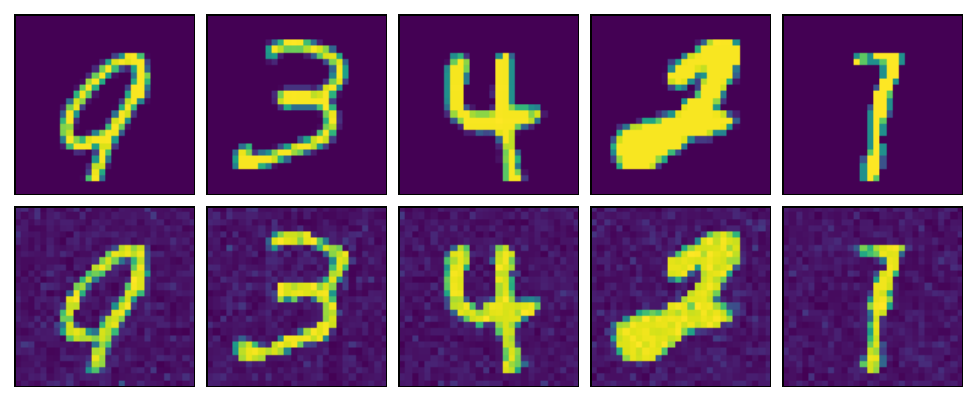}
        \caption{MNIST}
    \end{subfigure}
    \begin{subfigure}[t]{0.3\textwidth}
        \includegraphics[width=\columnwidth]{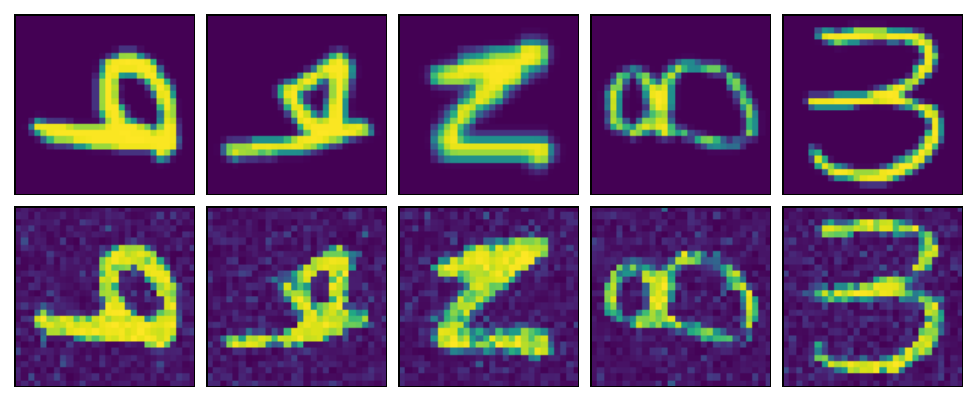}
        \caption{EMNIST}
    \end{subfigure}
    \begin{subfigure}[t]{0.3\textwidth}
    \includegraphics[width=\columnwidth]{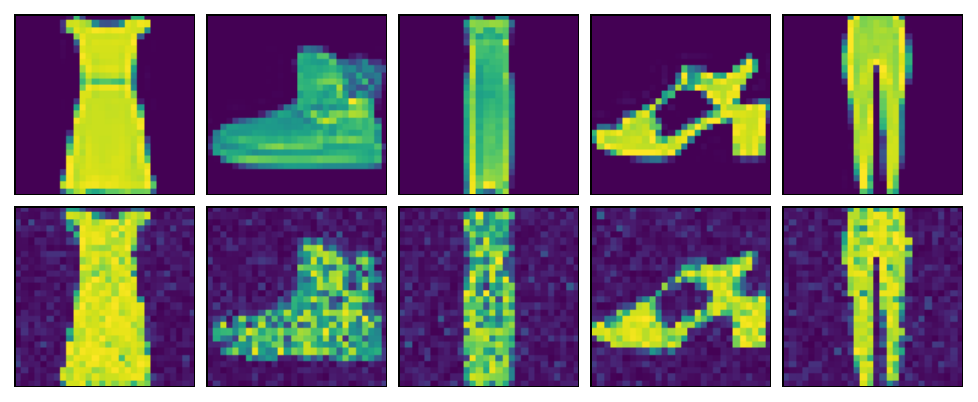}
        \caption{Fashion MNIST}
    \end{subfigure}
    \caption{\textbf{Top}: images from validation sets of MNIST (left), EMNIST (middle), and Fashion MNIST (right). \textbf{Bottom}: reconstructions by deep latent Gaussian models trained with DCPC for MNIST (left), EMNIST (middle), and Fashion MNIST (right), averaging over $K=4$ particles. DCPC achieves quality reconstructions by inference over $\mathbf{z}$ without training an inference network.}
    \label{fig:ppc_mnists_recons}
    \vspace{-1em}
\end{figure}

\paragraph{Image generation with representation learning}
\citet{Zahid2024} have also recently designed and evaluated Langevin predictive coding (LPC), with differences from both MCPC and DCPC. While MCPC sends prediction errors up through a hierarchical model, LPC computed as its prediction error the log-joint gradient for all latent variables in the generative model. This meant that biological plausibility, and their goal of amortizing predictive coding inference, restricted them to single-level decoder adapted from \citet{Higgins2017}. We evaluated with their discretized Gaussian likelihood, taken from \citet{Cheng2020,Ho2020}, learning the variance as in \citet{Rybkin2021}.

We compare DCPC to LPC using the Frechet Inception Distance (FID)~\citep{Seitzer2020} featured in \citet{Zahid2024}, holding constant the prior, neural network architecture, learning rate on $\theta$, and number of gradient evaluations used to train the parameters $\theta$ and latents $\mathbf{z}$. \citet{Zahid2024} evaluated a variety of scenarios and reported that their training could converge quickly when counted in epochs, but they accumulated gradients of $\theta$ over inference steps. We compare to the results they report after $15$ epochs with $300$ inference steps applied to latents initialized from the prior, equivalent to $15 \times 300 = 4500$ gradient steps on $\theta$ per batch, replicating their batch size of $64$. Since Algorithm~\ref{alg:ppc} updates $\theta$ only in its outer loop, we set $S=30$ and ran DCPC for $150$ epochs, with $\eta=0.001$ and $K=16$. Table~\ref{tab:celeba_results} shows that DCPC outperforms LPC in apples-to-apples generative quality, though not to the point of matching other model architectures\footnote{Such as vision Transformers, denoising diffusion models, etc.} by inference quality alone.

\begin{table}[t]
    \centering
    \begin{tabular}{l|c|l|l|l}
        Algorithm & Likelihood & Resolution $\uparrow$ & $S \times$ Epochs $\downarrow$ & FID $\downarrow$ \\
        \hline
        PGD & $\mathcal{N}$ & $32 \times 32$ & $1 \times 100$ &  $100 \pm 2.7$ \\
        DCPC (ours) & $\mathcal{N}$ & $32 \times 32$ & $1 \times 100$ & $\mathbf{82.7 \pm 0.9}$ \\
        \hline
        LPC & $\mathcal{DN}$ & $64 \times 64$ & $300 \times 15 = 4500$ & $120$ (approximate) \\
        VAE & $\mathcal{DN}$ & $64 \times 64$ & $1 \times 4500 = 4500$ & $86.3 \pm 0.03$ \\
        DCPC (ours) & $\mathcal{DN}$ & $64 \times 64$ & $30 \times 150 = 4500$ & $\mathbf{79.0} \pm 0.9$ \\
        \hline
    \end{tabular}
    \caption{FID score comparisons on the CelebA dataset~\citep{Liu2015}. The score for LPC comes from Figure 2 in \citet{Zahid2024}, where they ablated warm-starts and initialized from the prior.}
    \label{tab:celeba_results}
    \vspace{-2em}
\end{table}

Figure~\ref{fig:celeba} shows reconstructed images from the validation set (left) and samples from the posterior predictive generative model (right). There is blurriness in the reconstructions, as often occurs with variational autoencoders, but DCPC training allows the network to capture background color, hair color, direction in which a face is looking, and other visual properties.  Figure~\ref{fig:celeba_recons} shows reconstructions over the validation set, while Figure~\ref{fig:celeba_predictive} shows samples from the predictive distribution.

\citet{Kuntz2023} also reported an experiment on CelebA in terms of FID score, at the lower $32 \times 32$ resolution. Since they provided both source code and an exact mathematical description, we were able to run an exact, head-to-head comparison with PGD. The line in Table~\ref{tab:celeba_results} evaluating DCPC with PGD's example neural architecture at the $32 \times 32$ resolution (with similar particle count and learning rate) demonstrates a significant improvement in FID for DCPC, alongside reduced FID variance.


\begin{figure}
    \centering
    \begin{subfigure}{0.45\textwidth}
        \includegraphics[width=\columnwidth]{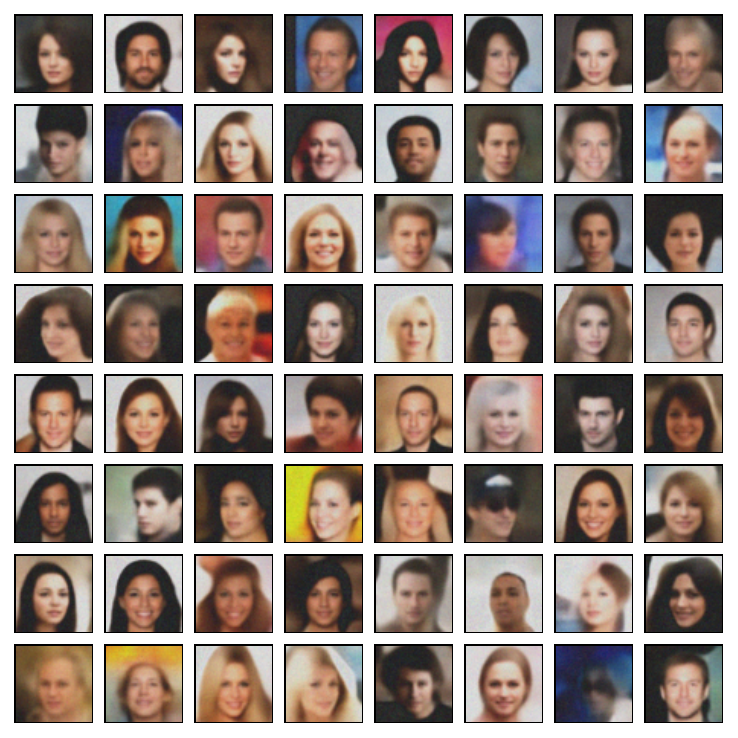}
        \caption{Reconstructions of the CelebA validation set by a generator network trained with DCPC.}
        \label{fig:celeba_recons}
    \end{subfigure}
    \begin{subfigure}{0.45\textwidth}
        \includegraphics[width=\columnwidth]{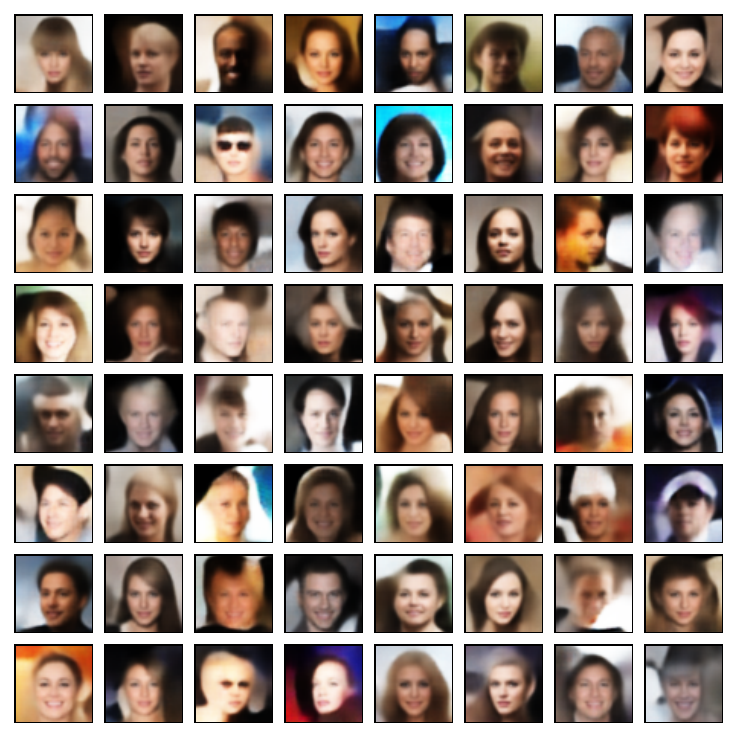}
        \caption{Samples drawn \emph{de novo} from the posterior predictive distribution of the trained network.}
        \label{fig:celeba_predictive}
    \end{subfigure}
    \caption{\textbf{Left}: reconstructions from the CelebA validation set. \textbf{Right}: samples from the generative model. DCPC achieves quality reconstructions by inference over $\mathbf{z}$ with $K=16$ particles and no inference network, while the learned generative model captures variation in the data.}
    \label{fig:celeba}
    \vspace{-1.5em}
\end{figure}

\paragraph{Necessary Compute Resources}
The initial DLGM experiments on the MNIST-alike datasets were performed on a desktop workstation with 128GB of RAM and an NVIDIA Quadro P4000 with 8GB of VRAM. Experiments on CelebA were conducted on an NVIDIA DGX equipped with eight (8) NVIDIA A100's, each with 80GB of VRAM. The latter compute infrastructure was also used for unpublished experiments, on several different datasets, in structured time-series modeling.

\vspace{-1em}
\section{Related Work}
\label{sec:related_work}
\vspace{-1em}

\citet{Pinchetti2022} expanded predictive coding beyond Gaussian generative models for the first time, applying the resulting algorithm to train variational autoencoders by variational inference and transformer architectures by maximum likelihood. DCPC, in turn, broadens predictive coding to target arbitrary probabilistic graphical models, following the broadening in \citet{Salvatori2022learning} to arbitrary deterministic computation graphs. DCPC follows on incremental predictive coding~\citep{Salvatori2024} in quickly alternating between updates to random variables and model parameters, giving an incremental EM algorithm~\citep{Neal1998}. Finally, \citet{Zahid2024} and \citet{Oliviers2024} also recognized the analogy between predictive coding's prediction errors and the score functions used in Langevin dynamics for continuous random variables.

There exists a large body of work on how neurobiologically plausible circuits could implement probabilistic inference. Classic work by \citet{Shi2009} provided a biologically plausible implementation of hierarchical inference via importance sampling; DCPC proceeds from importance sampling as a foundation, while parameterizing the proposal distribution via prediction errors. Recent work by \citet{Fang2022} studied neurally plausible algorithms for sampling-based inference with Langevin dynamics, though only for a Gaussian generative model of sparse coding. \citet{Golkar2022} imposed a whitening constraint on a Gaussian generative model for biological plausibility. Finally, \citet{Dong2023} and \citet{Zahid2024} both suggest mechanisms for employing momentum to reduce gradient noise in a biologically plausible sampling algorithm; the former intriguingly analogize their momentum term to neuronal adaptation. To conclude, other works have already implemented predictive coding models for image generation tasks, a notable example being the neural generative coding framework \cite{ororbia2022neural,ororbia2022convolutional}.

\vspace{-1em}
\section{Conclusion}
\label{sec:conclusion}
\vspace{-1em}

This paper proposed divide-and-conquer predictive coding (DCPC), an algorithm that efficiently and scalably approximates Gibbs samplers by importance sampling; DCPC parameterizes efficient proposals for a model's complete conditional densities using local prediction errors. Section~\ref{sec:bioplausibility} showed how Monte Carlo sampling can implement a form of ``prospective configuration''~\citep{Song2024}, first inferring a sample from the joint posterior density (Theorem~\ref{thm:ppc_complete_conditionals}) and then updating the generative model without a global backpropagation pass ( Theorem~\ref{thm:ppc_local_likelihood}). Experiments in Section~\ref{sec:experiments} showed that DCPC outperforms the state of the art Monte Carlo Predictive Coding from computational neuroscience, head-to-head, on the simple generative models typically considered in theoretical neuroscience; DCPC also outperforms the particle gradient descent algorithm of \citet{Kuntz2023} while under the constraint of purely local computation. DCPC's Langevin proposals admit the same extension to constrained sample spaces as applied in Hamiltonian Monte Carlo~\citep{Brubaker2012}; our Pyro implementation includes this extension via Pyro's preexisting support for HMC.

DCPC offers a number of ways forward. Particularly, this paper employed naive Langevin proposals, while \citet{Dong2023,Zahid2024} applied momentum-based preconditioning to take advantage of the target's geometry. \citet{Yin2006} demonstrated that gradient flows of this general kind can also provide more efficient samplers by breaking the detailed-balance condition necessary for the Metropolis-Hastings algorithm, motivating the choice of SMC over MCMC to correct proposal bias. 
Appendix~\ref{app:discrete_wasserstein} derives a mathematical background for an extension of DCPC to discrete random variables. Future work could follow \citet{Marino2018,Taniguchi2022} in using a neural network to iteratively map from particles and prediction errors to proposal parameters.

\vspace{-1em}
\subsection{Limitations}
\label{subsec:limitations}
\vspace{-1em}

DCPC's main limitations are its longer training time, and greater sensitivity to learning rates, than state-of-the-art amortized variational inference trained end-to-end. Such limitations occur frequently in the literature on neuroscience-inspired learning algorithms, as well as in the literature on particle-based algorithms with no parametric form.
This work has no singular ethical concerns specific only to DCPC, rather than the broader implications and responsibilities accompanying advancements in biologically plausible learning and Bayesian inference.


\begin{ack}
E.S. was supported by Vanderbilt Brain Institute Faculty Funds, as well as the Vanderbilt Data Science Institute, through PI Andre Bastos. Hamed Nejat produced the laminar circuitry figure in the supplementary material in conjunction with E.S. H.W. and T.S. were supported through VERSES.AI. E.S. would also like to thank Jan-Willem van de Meent, Karen Quigley, and Lisa Feldman Barrett for the training in approximate inference and predictive processing out of which this paper developed.
\end{ack}

\bibliography{neurips_2024}

\begin{thebibliography}{8}
\providecommand{\natexlab}[1]{#1}
\providecommand{\url}[1]{\texttt{#1}}
\expandafter\ifx\csname urlstyle\endcsname\relax
  \providecommand{\doi}[1]{doi: #1}\else
  \providecommand{\doi}{doi: \begingroup \urlstyle{rm}\Url}\fi

\bibitem[Chopin and Papaspiliopoulos(2020)]{Chopin2020}
Nicolas Chopin and Omiros Papaspiliopoulos.
\newblock \emph{An Introduction to Sequential Monte Carlo}.
\newblock Springer, 2020.
\newblock ISBN 978-3-030-47844-5.
\newblock \doi{10.1007/978-3-030-47845-2}.
\newblock Citation Key: Chopin2020 ISSN: 2197-568X.

\bibitem[Grathwohl et~al.(2021)Grathwohl, Swersky, Hashemi, Duvenaud, and Maddison]{Grathwohl2021}
Will Grathwohl, Kevin Swersky, Milad Hashemi, David Duvenaud, and Chris Maddison.
\newblock Oops {I} took a gradient: Scalable sampling for discrete distributions.
\newblock In \emph{Proceedings of the 38th International Conference on Machine Learning}, page 3831–3841. PMLR, July 2021.
\newblock URL \url{https://proceedings.mlr.press/v139/grathwohl21a.html}.

\bibitem[Schulman et~al.(2015)Schulman, Heess, Weber, and Abbeel]{Schulman2015}
John Schulman, Nicolas Heess, Theophane Weber, and Pieter Abbeel.
\newblock Gradient estimation using stochastic computation graphs.
\newblock In C.~Cortes, N.~Lawrence, D.~Lee, M.~Sugiyama, and R.~Garnett, editors, \emph{Advances in Neural Information Processing Systems}, volume~28. Curran Associates, Inc., 2015.
\newblock URL \url{https://proceedings.neurips.cc/paper_files/paper/2015/file/de03beffeed9da5f3639a621bcab5dd4-Paper.pdf}.

\bibitem[Sun et~al.(2023)Sun, Dai, Dai, Zhou, and Schuurmans]{Sun2023}
Haoran Sun, Hanjun Dai, Bo~Dai, Haomin Zhou, and Dale Schuurmans.
\newblock Discrete {Langevin} {Samplers} via {Wasserstein} {Gradient} {Flow}.
\newblock In \emph{Proceedings of the 26th {International} {Conference} on {Artificial} {Intelligence} and {Statistics}}, Valencia, Spain, April 2023. Proceedings of Machine Learning Research.

\bibitem[Titsias(2023)]{Titsias2023}
Michalis Titsias.
\newblock Optimal preconditioning and fisher adaptive langevin sampling.
\newblock In \emph{Advances in Neural Information Processing Systems}, volume~37. Curran Associates, Inc., 2023.
\newblock URL \url{https://proceedings.neurips.cc/paper_files/paper/2023/hash/5da6d5818a156791090c875abeca3cf8-Abstract-Conference.html}.

\bibitem[Wu et~al.(2020)Wu, Zimmermann, Sennesh, Le, and van~de Meent]{Wu2020}
Hao Wu, Heiko Zimmermann, Eli Sennesh, Tuan~Anh Le, and Jan~Willem van~de Meent.
\newblock {Amortized population Gibbs samplers with neural sufficient statistics}.
\newblock In \emph{Proceedings of the 37th International Conference on Machine Learning}, 2020.

\bibitem[Xiang et~al.(2023)Xiang, Zhu, Lei, Xu, and Zhang]{Xiang2023}
Yue Xiang, Dongyao Zhu, Bowen Lei, Dongkuan Xu, and Ruqi Zhang.
\newblock Efficient {Informed} {Proposals} for {Discrete} {Distributions} via {Newton}’s {Series} {Approximation}.
\newblock In \emph{Proceedings of {The} 26th {International} {Conference} on {Artificial} {Intelligence} and {Statistics}}, pages 7288--7310. PMLR, April 2023.
\newblock URL \url{https://proceedings.mlr.press/v206/xiang23a.html}.
\newblock ISSN: 2640-3498.

\bibitem[Zimmermann et~al.(2021)Zimmermann, Wu, Esmaeili, and van~de Meent]{Zimmermann2021}
Heiko Zimmermann, Hao Wu, Babak Esmaeili, and Jan-Willem van~de Meent.
\newblock Nested variational inference.
\newblock In M.~Ranzato, A.~Beygelzimer, Y.~Dauphin, P.S. Liang, and J.~Wortman Vaughan, editors, \emph{Advances in Neural Information Processing Systems}, volume~34, pages 20423--20435. Curran Associates, Inc., 2021.
\newblock URL \url{https://proceedings.neurips.cc/paper_files/paper/2021/file/ab49b208848abe14418090d95df0d590-Paper.pdf}.

\end{thebibliography}


\begin{thebibliography}{74}
\providecommand{\natexlab}[1]{#1}
\providecommand{\url}[1]{\texttt{#1}}
\expandafter\ifx\csname urlstyle\endcsname\relax
  \providecommand{\doi}[1]{doi: #1}\else
  \providecommand{\doi}{doi: \begingroup \urlstyle{rm}\Url}\fi

\bibitem[Adler and Alon(2018)]{Adler2018}
Miri Adler and Uri Alon.
\newblock Fold-change detection in biological systems.
\newblock \emph{Current Opinion in Systems Biology}, 8:\penalty0 81–89, April 2018.
\newblock ISSN 2452-3100.
\newblock \doi{10.1016/j.coisb.2017.12.005}.

\bibitem[Bastos et~al.(2015)Bastos, Vezoli, and Fries]{Bastos2015}
Andre~M Bastos, Julien Vezoli, and Pascal Fries.
\newblock Communication through coherence with inter-areal delays.
\newblock \emph{Current Opinion in Neurobiology}, 31:\penalty0 173–180, April 2015.
\newblock ISSN 09594388.
\newblock \doi{10.1016/j.conb.2014.11.001}.

\bibitem[Bastos et~al.(2012)Bastos, Usrey, Adams, Mangun, Fries, and Friston]{Bastos2012}
André~M. Bastos, W~Martin Usrey, Rick~A Adams, George~R Mangun, Pascal Fries, and Karl~J Friston.
\newblock Canonical microcircuits for predictive coding.
\newblock \emph{Neuron}, 76\penalty0 (4):\penalty0 695--711, 2012.

\bibitem[Bastos et~al.(2020)Bastos, Lundqvist, Waite, Kopell, and Miller]{Bastos2020}
André~M. Bastos, Mikael Lundqvist, Ayan~S. Waite, Nancy Kopell, and Earl~K. Miller.
\newblock Layer and rhythm specificity for predictive routing.
\newblock \emph{Proceedings of the National Academy of Sciences}, 117\penalty0 (49):\penalty0 31459–31469, December 2020.
\newblock ISSN 0027-8424, 1091-6490.
\newblock \doi{10.1073/pnas.2014868117}.

\bibitem[Bergmann et~al.(2024)Bergmann, Petro, Abbatecola, Li, Morgan, and Muckli]{Bergmann2024}
Johanna Bergmann, Lucy~S Petro, Clement Abbatecola, Min~S Li, A~Tyler Morgan, and Lars Muckli.
\newblock Cortical depth profiles in primary visual cortex for illusory and imaginary experiences.
\newblock \emph{Nature Communications}, 15\penalty0 (1):\penalty0 1002, 2024.

\bibitem[Bingham et~al.(2019)Bingham, Chen, Jankowiak, Obermeyer, Pradhan, Karaletsos, Singh, Szerlip, Horsfall, and Goodman]{Bingham2019}
Eli Bingham, Jonathan~P. Chen, Martin Jankowiak, Fritz Obermeyer, Neeraj Pradhan, Theofanis Karaletsos, Rohit Singh, Paul Szerlip, Paul Horsfall, and Noah~D. Goodman.
\newblock Pyro: Deep universal probabilistic programming.
\newblock \emph{Journal of Machine Learning Research}, 20\penalty0 (28):\penalty0 1--6, 2019.
\newblock URL \url{http://jmlr.org/papers/v20/18-403.html}.

\bibitem[Borba et~al.(2021)Borba, Kourakis, Schwennicke, Brasnic, and Smith]{Borba2021}
Cezar Borba, Matthew~J Kourakis, Shea Schwennicke, Lorena Brasnic, and William~C Smith.
\newblock Fold change detection in visual processing.
\newblock \emph{Frontiers in Neural Circuits}, 15:\penalty0 705161, 2021.

\bibitem[Brubaker et~al.(2012)Brubaker, Salzmann, and Urtasun]{Brubaker2012}
Marcus Brubaker, Mathieu Salzmann, and Raquel Urtasun.
\newblock A family of mcmc methods on implicitly defined manifolds.
\newblock In \emph{Artificial intelligence and statistics}, pages 161--172. PMLR, 2012.

\bibitem[Buzsáki and Draguhn(2004)]{Buzsaki2004}
György Buzsáki and Andreas Draguhn.
\newblock Neuronal oscillations in cortical networks.
\newblock \emph{Science}, 304\penalty0 (5679):\penalty0 1926–1929, June 2004.
\newblock ISSN 0036-8075, 1095-9203.
\newblock \doi{10.1126/science.1099745}.

\bibitem[Campagnola et~al.(2022)Campagnola, Seeman, Chartrand, Kim, Hoggarth, Gamlin, Ito, Trinh, Davoudian, Radaelli, Kim, Hage, Braun, Alfiler, Andrade, Bohn, Dalley, Henry, Kebede, Mukora, Sandman, Williams, Larsen, Teeter, Daigle, Berry, Dotson, Enstrom, Gorham, Hupp, Lee, Ngo, Nicovich, Potekhina, Ransford, Gary, Goldy, McMillen, Pham, Tieu, Siverts, Walker, Farrell, Schroedter, Slaughterbeck, Cobb, Ellenbogen, Gwinn, Keene, Ko, Ojemann, Silbergeld, Carey, Casper, Crichton, Clark, Dee, Ellingwood, Gloe, Kroll, Sulc, Tung, Wadhwani, Brouner, Egdorf, Maxwell, McGraw, Pom, Ruiz, Bomben, Feng, Hejazinia, Shi, Szafer, Wakeman, Phillips, Bernard, Esposito, D’Orazi, Sunkin, Smith, Tasic, Arkhipov, Sorensen, Lein, Koch, Murphy, Zeng, and Jarsky]{Campagnola2022}
Luke Campagnola, Stephanie~C. Seeman, Thomas Chartrand, Lisa Kim, Alex Hoggarth, Clare Gamlin, Shinya Ito, Jessica Trinh, Pasha Davoudian, Cristina Radaelli, Mean-Hwan Kim, Travis Hage, Thomas Braun, Lauren Alfiler, Julia Andrade, Phillip Bohn, Rachel Dalley, Alex Henry, Sara Kebede, Alice Mukora, David Sandman, Grace Williams, Rachael Larsen, Corinne Teeter, Tanya~L. Daigle, Kyla Berry, Nadia Dotson, Rachel Enstrom, Melissa Gorham, Madie Hupp, Samuel~Dingman Lee, Kiet Ngo, Philip~R. Nicovich, Lydia Potekhina, Shea Ransford, Amanda Gary, Jeff Goldy, Delissa McMillen, Trangthanh Pham, Michael Tieu, La’Akea Siverts, Miranda Walker, Colin Farrell, Martin Schroedter, Cliff Slaughterbeck, Charles Cobb, Richard Ellenbogen, Ryder~P. Gwinn, C.~Dirk Keene, Andrew~L. Ko, Jeffrey~G. Ojemann, Daniel~L. Silbergeld, Daniel Carey, Tamara Casper, Kirsten Crichton, Michael Clark, Nick Dee, Lauren Ellingwood, Jessica Gloe, Matthew Kroll, Josef Sulc, Herman Tung, Katherine Wadhwani, Krissy Brouner, Tom Egdorf, Michelle
  Maxwell, Medea McGraw, Christina~Alice Pom, Augustin Ruiz, Jasmine Bomben, David Feng, Nika Hejazinia, Shu Shi, Aaron Szafer, Wayne Wakeman, John Phillips, Amy Bernard, Luke Esposito, Florence~D. D’Orazi, Susan Sunkin, Kimberly Smith, Bosiljka Tasic, Anton Arkhipov, Staci Sorensen, Ed~Lein, Christof Koch, Gabe Murphy, Hongkui Zeng, and Tim Jarsky.
\newblock Local connectivity and synaptic dynamics in mouse and human neocortex.
\newblock \emph{Science}, 375\penalty0 (6585):\penalty0 eabj5861, 2022.
\newblock \doi{10.1126/science.abj5861}.
\newblock URL \url{https://www.science.org/doi/abs/10.1126/science.abj5861}.

\bibitem[Carandini and Heeger(2012)]{Carandini2012}
Matteo Carandini and David~J Heeger.
\newblock Normalization as a canonical neural computation.
\newblock \emph{Nature reviews neuroscience}, 13\penalty0 (1):\penalty0 51--62, 2012.

\bibitem[Chater et~al.(2006)Chater, Tenenbaum, and Yuille]{Chater2006}
Nick Chater, Joshua~B Tenenbaum, and Alan Yuille.
\newblock Probabilistic models of cognition: Conceptual foundations.
\newblock \emph{Trends in cognitive sciences}, 10\penalty0 (7):\penalty0 287--291, 2006.

\bibitem[Chavlis and Poirazi(2021)]{Chavlis2021}
Spyridon Chavlis and Panayiota Poirazi.
\newblock Drawing inspiration from biological dendrites to empower artificial neural networks.
\newblock \emph{Current Opinion in Neurobiology}, 70:\penalty0 1–10, October 2021.
\newblock ISSN 0959-4388.
\newblock \doi{10.1016/j.conb.2021.04.007}.

\bibitem[Cheng et~al.(2020)Cheng, Sun, Takeuchi, and Katto]{Cheng2020}
Zhengxue Cheng, Heming Sun, Masaru Takeuchi, and Jiro Katto.
\newblock Learned image compression with discretized gaussian mixture likelihoods and attention modules.
\newblock In \emph{Proceedings of the 2020 IEEE/CVF Conference on Computer Vision and Pattern Recognition (CVPR)}, 2020.

\bibitem[Cohen et~al.(2017)Cohen, Afshar, Tapson, and Van~Schaik]{cohen2017emnist}
Gregory Cohen, Saeed Afshar, Jonathan Tapson, and Andre Van~Schaik.
\newblock Emnist: Extending mnist to handwritten letters.
\newblock In \emph{2017 international joint conference on neural networks (IJCNN)}, pages 2921--2926. IEEE, 2017.

\bibitem[Dasgupta et~al.(2020)Dasgupta, Schulz, Tenenbaum, and Gershman]{Dasgupta2020}
Ishita Dasgupta, Eric Schulz, Joshua~B Tenenbaum, and Samuel~J Gershman.
\newblock A theory of learning to infer.
\newblock \emph{Psychological review}, 127\penalty0 (3):\penalty0 412, 2020.

\bibitem[Dong and Wu(2023)]{Dong2023}
Xingsi Dong and Si~Wu.
\newblock Neural {Sampling} in {Hierarchical} {Exponential}-family {Energy}-based {Models}.
\newblock In \emph{Advances in {Neural} {Information} {Processing} {Systems}}, New Orleans, LA, 2023. Curran Associates Inc.

\bibitem[Douglas and Martin(2004)]{Douglas2004}
Rodney~J Douglas and Kevan~AC Martin.
\newblock Neuronal circuits of the neocortex.
\newblock \emph{Annu. Rev. Neurosci.}, 27\penalty0 (1):\penalty0 419--451, 2004.

\bibitem[Fang et~al.(2022)Fang, Mudigonda, Zarcone, Khosrowshahi, and Olshausen]{Fang2022}
Michael Y-S Fang, Mayur Mudigonda, Ryan Zarcone, Amir Khosrowshahi, and Bruno~A Olshausen.
\newblock Learning and inference in sparse coding models with langevin dynamics.
\newblock \emph{Neural Computation}, 34\penalty0 (8):\penalty0 1676--1700, 2022.

\bibitem[Forkosh(2022)]{Forkosh2022}
Oren Forkosh.
\newblock Memoryless optimality: Neurons do not need adaptation to optimally encode stimuli with arbitrarily complex statistics.
\newblock \emph{Neural Computation}, 34\penalty0 (12):\penalty0 2374–2387, November 2022.
\newblock ISSN 0899-7667, 1530-888X.
\newblock \doi{10.1162/neco_a_01543}.

\bibitem[Fries(2015)]{Fries2015}
Pascal Fries.
\newblock Rhythms for cognition: Communication through coherence.
\newblock \emph{Neuron}, 88\penalty0 (1):\penalty0 220–235, October 2015.
\newblock ISSN 0896-6273.
\newblock \doi{10.1016/j.neuron.2015.09.034}.

\bibitem[Friston(2005)]{Friston2005}
Karl Friston.
\newblock A theory of cortical responses.
\newblock \emph{Philosophical transactions of the Royal Society B: Biological sciences}, 360\penalty0 (1456):\penalty0 815--836, 2005.

\bibitem[Gillon et~al.(2023)Gillon, Pina, Lecoq, Ahmed, Billeh, Caldejon, Groblewski, Henley, Kato, Lee, Luviano, Mace, Nayan, Nguyen, North, Perkins, Seid, Valley, Williford, Bengio, Lillicrap, Richards, and Zylberberg]{Gillon2023}
Colleen~J. Gillon, Jason~E. Pina, J{\'e}r{\^o}me~A. Lecoq, Ruweida Ahmed, Yazan~N. Billeh, Shiella Caldejon, Peter Groblewski, Timothy~M. Henley, India Kato, Eric Lee, Jennifer Luviano, Kyla Mace, Chelsea Nayan, Thuyanh~V. Nguyen, Kat North, Jed Perkins, Sam Seid, Matthew~T. Valley, Ali Williford, Yoshua Bengio, Timothy~P. Lillicrap, Blake~A. Richards, and Joel Zylberberg.
\newblock Learning from unexpected events in the neocortical microcircuit.
\newblock \emph{bioRxiv}, 2023.
\newblock \doi{10.1101/2021.01.15.426915}.
\newblock URL \url{https://www.biorxiv.org/content/early/2023/04/06/2021.01.15.426915}.

\bibitem[Golkar et~al.(2022)Golkar, Tesileanu, Bahroun, Sengupta, and Chklovskii]{Golkar2022}
Siavash Golkar, Tiberiu Tesileanu, Yanis Bahroun, Anirvan Sengupta, and Dmitri Chklovskii.
\newblock Constrained predictive coding as a biologically plausible model of the cortical hierarchy.
\newblock In S.~Koyejo, S.~Mohamed, A.~Agarwal, D.~Belgrave, K.~Cho, and A.~Oh, editors, \emph{Advances in Neural Information Processing Systems}, volume~35, pages 14155--14169. Curran Associates, Inc., 2022.
\newblock URL \url{https://proceedings.neurips.cc/paper_files/paper/2022/file/5b5de8526aac159e37ff9547713677ed-Paper-Conference.pdf}.

\bibitem[Gonzalez et~al.(2011)Gonzalez, Low, Gretton, and Guestrin]{Gonzalez2011}
Joseph Gonzalez, Yucheng Low, Arthur Gretton, and Carlos Guestrin.
\newblock Parallel gibbs sampling: From colored fields to thin junction trees.
\newblock In Geoffrey Gordon, David Dunson, and Miroslav Dudík, editors, \emph{Proceedings of the Fourteenth International Conference on Artificial Intelligence and Statistics}, volume~15 of \emph{Proceedings of Machine Learning Research}, pages 324--332, Fort Lauderdale, FL, USA, 11--13 Apr 2011. PMLR.
\newblock URL \url{https://proceedings.mlr.press/v15/gonzalez11a.html}.

\bibitem[Higgins et~al.(2017)Higgins, Matthey, Pal, Burgess, Glorot, Botvinick, Mohamed, and Lerchner]{Higgins2017}
Irina Higgins, Loic Matthey, Arka Pal, Christopher Burgess, Xavier Glorot, Matthew Botvinick, Shakir Mohamed, and Alexander Lerchner.
\newblock beta-{VAE}: Learning basic visual concepts with a constrained variational framework.
\newblock In \emph{International Conference on Learning Representations}, 2017.
\newblock URL \url{https://openreview.net/forum?id=Sy2fzU9gl}.

\bibitem[Ho et~al.(2020)Ho, Jain, and Abbeel]{Ho2020}
Jonathan Ho, Ajay Jain, and Pieter Abbeel.
\newblock Denoising diffusion probabilistic models.
\newblock In \emph{Advances in Neural Information Processing Systems}, volume~33, page 6840–6851. Curran Associates, Inc., 2020.
\newblock URL \url{https://proceedings.neurips.cc/paper/2020/hash/4c5bcfec8584af0d967f1ab10179ca4b-Abstract.html}.

\bibitem[Hutchinson and Barrett(2019)]{Hutchinson2019}
J.~Benjamin Hutchinson and Lisa~Feldman Barrett.
\newblock {The Power of Predictions: An Emerging Paradigm for Psychological Research}.
\newblock \emph{Current Directions in Psychological Science}, 2019.
\newblock ISSN 14678721.
\newblock \doi{10.1177/0963721419831992}.

\bibitem[Kingma and Ba(2014)]{kingma2014adam}
Diederik~P Kingma and Jimmy Ba.
\newblock Adam: A method for stochastic optimization.
\newblock \emph{arXiv preprint arXiv:1412.6980}, 2014.

\bibitem[Kingma and Welling(2013)]{kingma2013auto}
Diederik~P Kingma and Max Welling.
\newblock Auto-encoding variational bayes.
\newblock \emph{arXiv preprint arXiv:1312.6114}, 2013.

\bibitem[Kuntz et~al.(2023)Kuntz, Lim, and Johansen]{Kuntz2023}
Juan Kuntz, Jen~Ning Lim, and Adam~M Johansen.
\newblock Particle algorithms for maximum likelihood training of latent variable models.
\newblock In \emph{Proceedings of the 26th {International} {Conference} on {Artificial} {Intelligence} and {Statistics}}, volume 206, Valencia, Spain, April 2023. Proceedings of Machine Learning Research.

\bibitem[Kuntz et~al.(2024)Kuntz, Crucinio, and Johansen]{Kuntz2024}
Juan Kuntz, Francesca~R. Crucinio, and Adam~M. Johansen.
\newblock {The divide-and-conquer sequential Monte Carlo algorithm: Theoretical properties and limit theorems}.
\newblock \emph{The Annals of Applied Probability}, 34\penalty0 (1B):\penalty0 1469 -- 1523, 2024.
\newblock \doi{10.1214/23-AAP1996}.
\newblock URL \url{https://doi.org/10.1214/23-AAP1996}.

\bibitem[Lake et~al.(2017)Lake, Ullman, Tenenbaum, and Gershman]{Lake2017}
Brenden~M. Lake, Tomer~D. Ullman, Joshua~B. Tenenbaum, and Samuel~J. Gershman.
\newblock Building machines that learn and think like people.
\newblock \emph{Behavioral and Brain Sciences}, 40:\penalty0 e253, 2017.
\newblock \doi{10.1017/S0140525X16001837}.

\bibitem[Lecun et~al.(2015)Lecun, Bengio, and Hinton]{Lecun2015}
Yann Lecun, Yoshua Bengio, and Geoffrey Hinton.
\newblock Deep learning.
\newblock \emph{Nature}, 521\penalty0 (7553):\penalty0 436–444, 2015.
\newblock ISSN 14764687.
\newblock \doi{10.1038/nature14539}.
\newblock Citation Key: Lecun2015.

\bibitem[Lindsten et~al.(2017)Lindsten, Johansen, Naesseth, Kirkpatrick, Schön, Aston, and Bouchard-Côté]{Lindsten2017}
F.~Lindsten, A.~M. Johansen, C.~A. Naesseth, B.~Kirkpatrick, T.~B. Schön, J.~A.D. Aston, and A.~Bouchard-Côté.
\newblock Divide-and-conquer with sequential monte carlo.
\newblock \emph{Journal of Computational and Graphical Statistics}, 26\penalty0 (2):\penalty0 445–458, 2017.
\newblock ISSN 15372715.
\newblock \doi{10.1080/10618600.2016.1237363}.
\newblock arXiv: 1406.4993 Citation Key: Lindsten2017.

\bibitem[Liu(1999)]{Liu1999}
Shih-Chii Liu.
\newblock A winner-take-all circuit with controllable soft max property.
\newblock In S.~Solla, T.~Leen, and K.~M\"{u}ller, editors, \emph{Advances in Neural Information Processing Systems}, volume~12. MIT Press, 1999.
\newblock URL \url{https://proceedings.neurips.cc/paper_files/paper/1999/file/3e7e0224018ab3cf51abb96464d518cd-Paper.pdf}.

\bibitem[Liu et~al.(2015)Liu, Luo, Wang, and Tang]{Liu2015}
Ziwei Liu, Ping Luo, Xiaogang Wang, and Xiaoou Tang.
\newblock Deep learning face attributes in the wild.
\newblock In \emph{Proceedings of International Conference on Computer Vision (ICCV)}, December 2015.

\bibitem[Loaiza-Ganem and Cunningham(2019)]{Loaiza2019}
Gabriel Loaiza-Ganem and John~P Cunningham.
\newblock The continuous bernoulli: fixing a pervasive error in variational autoencoders.
\newblock In \emph{Advances in Neural Information Processing Systems}, volume~32. Curran Associates, Inc., 2019.
\newblock URL \url{https://proceedings.neurips.cc/paper/2019/hash/f82798ec8909d23e55679ee26bb26437-Abstract.html}.

\bibitem[Lundstrom et~al.(2008)Lundstrom, Higgs, Spain, and Fairhall]{Lundstrom2008}
Brian~N. Lundstrom, Matthew~H. Higgs, William~J. Spain, and Adrienne~L. Fairhall.
\newblock Fractional differentiation by neocortical pyramidal neurons.
\newblock \emph{Nature Neuroscience}, 11\penalty0 (11):\penalty0 1335–1342, November 2008.
\newblock ISSN 1546-1726.
\newblock \doi{10.1038/nn.2212}.

\bibitem[Mainen and Sejnowski(1995)]{Mainen1995}
Zachary~F Mainen and Terrence~J Sejnowski.
\newblock Reliability of spike timing in neocortical neurons.
\newblock \emph{Science}, 268\penalty0 (5216):\penalty0 1503--1506, 1995.

\bibitem[Marino et~al.(2018)Marino, Yue, and Mandt]{Marino2018}
Joseph Marino, Yisong Yue, and Stephan Mandt.
\newblock Iterative amortized inference.
\newblock In \emph{35th International Conference on Machine Learning, ICML 2018}, volume~8, page 5444–5462, 2018.
\newblock ISBN 978-1-5108-6796-3.
\newblock arXiv: 1807.09356 Citation Key: Marino2018a.

\bibitem[Millidge et~al.(2021)Millidge, Seth, and Buckley]{Millidge2021}
Beren Millidge, Anil Seth, and Christopher~L Buckley.
\newblock Predictive coding: a theoretical and experimental review.
\newblock \emph{arXiv preprint arXiv:2107.12979}, 2021.

\bibitem[Millidge et~al.(2023)Millidge, Song, Salvatori, Lukasiewicz, and Bogacz]{Millidge2023}
Beren Millidge, Yuhang Song, Tommaso Salvatori, Thomas Lukasiewicz, and Rafal Bogacz.
\newblock A theoretical framework for inference and learning in predictive coding networks.
\newblock In \emph{The Eleventh International Conference on Learning Representations}, 2023.
\newblock URL \url{https://openreview.net/forum?id=ZCTvSF_uVM4}.

\bibitem[Moldwin et~al.(2021)Moldwin, Kalmenson, and Segev]{Moldwin2021}
Toviah Moldwin, Menachem Kalmenson, and Idan Segev.
\newblock The gradient clusteron: A model neuron that learns to solve classification tasks via dendritic nonlinearities, structural plasticity, and gradient descent.
\newblock \emph{PLOS Computational Biology}, 17\penalty0 (5):\penalty0 e1009015, May 2021.
\newblock ISSN 1553-7358.
\newblock \doi{10.1371/journal.pcbi.1009015}.

\bibitem[Naesseth et~al.(2015)Naesseth, Lindsten, and Schon]{Naesseth2015}
Christian Naesseth, Fredrik Lindsten, and Thomas Schon.
\newblock Nested sequential monte carlo methods.
\newblock In Francis Bach and David Blei, editors, \emph{Proceedings of the 32nd International Conference on Machine Learning}, volume~37 of \emph{Proceedings of Machine Learning Research}, pages 1292--1301, Lille, France, 07--09 Jul 2015. PMLR.
\newblock URL \url{https://proceedings.mlr.press/v37/naesseth15.html}.

\bibitem[Naesseth et~al.(2018)Naesseth, Linderman, Ranganath, and Blei]{Naesseth2018}
Christian Naesseth, Scott Linderman, Rajesh Ranganath, and David Blei.
\newblock Variational sequential monte carlo.
\newblock In Amos Storkey and Fernando Perez-Cruz, editors, \emph{Proceedings of the Twenty-First International Conference on Artificial Intelligence and Statistics}, volume~84 of \emph{Proceedings of Machine Learning Research}, pages 968--977. PMLR, 09--11 Apr 2018.
\newblock URL \url{https://proceedings.mlr.press/v84/naesseth18a.html}.

\bibitem[Neal and Hinton(1998)]{Neal1998}
Radford~M. Neal and Geoffrey~E. Hinton.
\newblock \emph{A View of the EM Algorithm that Justifies Incremental, Sparse, and other Variants}, page 355–368.
\newblock NATO ASI Series. Springer Netherlands, Dordrecht, 1998.
\newblock ISBN 978-94-011-5014-9.
\newblock \doi{10.1007/978-94-011-5014-9_12}.
\newblock URL \url{https://doi.org/10.1007/978-94-011-5014-9_12}.

\bibitem[Oliviers et~al.(2024)Oliviers, Bogacz, and Meulemans]{Oliviers2024}
Gaspard Oliviers, Rafal Bogacz, and Alexander Meulemans.
\newblock Learning probability distributions of sensory inputs with monte carlo predictive coding.
\newblock \emph{bioRxiv}, 2024.

\bibitem[Ororbia and Kifer(2022)]{ororbia2022neural}
Alexander Ororbia and Daniel Kifer.
\newblock The neural coding framework for learning generative models.
\newblock \emph{Nature communications}, 13\penalty0 (1):\penalty0 2064, 2022.

\bibitem[Ororbia and Mali(2022)]{ororbia2022convolutional}
Alexander Ororbia and Ankur Mali.
\newblock Convolutional neural generative coding: Scaling predictive coding to natural images.
\newblock \emph{arXiv preprint arXiv:2211.12047}, 2022.

\bibitem[Peters et~al.(2024)Peters, DiCarlo, Gureckis, Haefner, Isik, Tenenbaum, Konkle, Naselaris, Stachenfeld, Tavares, et~al.]{Peters2024}
Benjamin Peters, James~J DiCarlo, Todd Gureckis, Ralf Haefner, Leyla Isik, Joshua Tenenbaum, Talia Konkle, Thomas Naselaris, Kimberly Stachenfeld, Zenna Tavares, et~al.
\newblock How does the primate brain combine generative and discriminative computations in vision?
\newblock \emph{ArXiv}, 2024.

\bibitem[Pinchetti et~al.(2022)Pinchetti, Salvatori, Yordanov, Millidge, Song, and Lukasiewicz]{Pinchetti2022}
Luca Pinchetti, Tommaso Salvatori, Yordan Yordanov, Beren Millidge, Yuhang Song, and Thomas Lukasiewicz.
\newblock Predictive coding beyond gaussian distributions.
\newblock In S.~Koyejo, S.~Mohamed, A.~Agarwal, D.~Belgrave, K.~Cho, and A.~Oh, editors, \emph{Advances in Neural Information Processing Systems}, volume~35, pages 1280--1293. Curran Associates, Inc., 2022.
\newblock URL \url{https://proceedings.neurips.cc/paper_files/paper/2022/file/08f9de0232c0b485110237f6e6cf88f1-Paper-Conference.pdf}.

\bibitem[Pouget et~al.(2013)Pouget, Beck, Ma, and Latham]{Pouget2013}
Alexandre Pouget, Jeffrey~M Beck, Wei~Ji Ma, and Peter~E Latham.
\newblock Probabilistic brains: knowns and unknowns.
\newblock \emph{Nature neuroscience}, 16\penalty0 (9):\penalty0 1170--1178, 2013.

\bibitem[Rao and Ballard(1999)]{Rao1999}
Rajesh P~N Rao and Dana~H Ballard.
\newblock {Predictive coding in the visual cortex: a functional interpretation of some extra-classical receptive-field effects}.
\newblock \emph{Nature neuroscience}, 2\penalty0 (1):\penalty0 79--87, 1999.
\newblock ISSN 1097-6256.
\newblock \doi{10.1038/4580}.
\newblock URL \url{10.1038/4580%5Cnhttp://www.nature.com/neuro/journal/v2/n1/abs/nn0199_79.html}.

\bibitem[Rezende et~al.(2014)Rezende, Mohamed, and Wierstra]{Rezende2014}
Danilo~Jimenez Rezende, Shakir Mohamed, and Daan Wierstra.
\newblock Stochastic backpropagation and approximate inference in deep generative models.
\newblock In \emph{Proceedings of the 31st International Conference on Machine Learning}, volume~4, page 3057–3070, Beijing, China, 2014.
\newblock arXiv: 1401.4082 Citation Key: Rezende2014 ISBN: 9781634393973.

\bibitem[Rumelhart et~al.(1986)Rumelhart, Hinton, and Williams]{Rumelhart1986}
David~E Rumelhart, Geoffrey~E Hinton, and Ronald~J Williams.
\newblock Learning representations by back-propagating errors.
\newblock \emph{Nature}, 323\penalty0 (6088):\penalty0 533--536, 1986.

\bibitem[Rybkin et~al.(2021)Rybkin, Daniilidis, and Levine]{Rybkin2021}
Oleh Rybkin, Kostas Daniilidis, and Sergey Levine.
\newblock Simple and effective vae training with calibrated decoders.
\newblock In \emph{International conference on machine learning}, pages 9179--9189. PMLR, 2021.

\bibitem[Salvatori et~al.(2022)Salvatori, Pinchetti, Millidge, Song, Bao, Bogacz, and Lukasiewicz]{Salvatori2022learning}
Tommaso Salvatori, Luca Pinchetti, Beren Millidge, Yuhang Song, Tianyi Bao, Rafal Bogacz, and Thomas Lukasiewicz.
\newblock Learning on arbitrary graph topologies via predictive coding.
\newblock \emph{Advances in neural information processing systems}, 35:\penalty0 38232--38244, 2022.

\bibitem[Salvatori et~al.(2023)Salvatori, Mali, Buckley, Lukasiewicz, Rao, Friston, and Ororbia]{Salvatori2023}
Tommaso Salvatori, Ankur Mali, Christopher~L Buckley, Thomas Lukasiewicz, Rajesh~PN Rao, Karl Friston, and Alexander Ororbia.
\newblock Brain-inspired computational intelligence via predictive coding.
\newblock \emph{arXiv preprint arXiv:2308.07870}, 2023.

\bibitem[Salvatori et~al.(2024)Salvatori, Song, Yordanov, Millidge, Emde, Xu, Sha, Bogacz, and Lukasiewicz]{Salvatori2024}
Tommaso Salvatori, Yuhang Song, Yordan Yordanov, Beren Millidge, Cornelius Emde, Zhenghua Xu, Lei Sha, Rafal Bogacz, and Thomas Lukasiewicz.
\newblock A stable, fast, and fully automatic learning algorithm for predictive coding networks.
\newblock In \emph{International Conference on Learning Representations}, 2024.

\bibitem[Schmidhuber(2015)]{Schmidhuber2015}
Jürgen Schmidhuber.
\newblock Deep learning in neural networks: An overview.
\newblock \emph{Neural Networks}, 61:\penalty0 85–117, 2015.
\newblock ISSN 18792782.
\newblock \doi{10.1016/j.neunet.2014.09.003}.
\newblock Citation Key: Schmidhuber2015.

\bibitem[Seitzer(2020)]{Seitzer2020}
Maximilian Seitzer.
\newblock {pytorch-fid: FID Score for PyTorch}.
\newblock \url{https://github.com/mseitzer/pytorch-fid}, August 2020.
\newblock Version 0.3.0.

\bibitem[Shi and Griffiths(2009)]{Shi2009}
Lei Shi and Thomas~L. Griffiths.
\newblock Neural implementation of hierarchical bayesian inference by importance sampling.
\newblock In \emph{Advances in Neural Information Processing Systems}, page 1669–1677, 2009.
\newblock ISBN 978-1-61567-911-9.
\newblock Citation Key: Shi2009.

\bibitem[Song et~al.(2024)Song, Millidge, Salvatori, Lukasiewicz, Xu, and Bogacz]{Song2024}
Yuhang Song, Beren Millidge, Tommaso Salvatori, Thomas Lukasiewicz, Zhenghua Xu, and Rafal Bogacz.
\newblock Inferring neural activity before plasticity as a foundation for learning beyond backpropagation.
\newblock \emph{Nature Neuroscience}, page 1–11, January 2024.
\newblock ISSN 1546-1726.
\newblock \doi{10.1038/s41593-023-01514-1}.

\bibitem[Spratling(2017)]{Spratling2017}
M.~W. Spratling.
\newblock {A review of predictive coding algorithms}.
\newblock \emph{Brain and Cognition}, 112:\penalty0 92--97, 2017.
\newblock ISSN 10902147.
\newblock \doi{10.1016/j.bandc.2015.11.003}.

\bibitem[Srinivasan et~al.(1982)Srinivasan, Laughlin, and Dubs]{Srinivasan1982}
M.~V. Srinivasan, S.~B. Laughlin, and A.~Dubs.
\newblock {Predictive coding: A fresh view of inhibition in the retina}.
\newblock \emph{Proceedings of the Royal Society of London - Biological Sciences}, 216\penalty0 (1205):\penalty0 427--459, 1982.
\newblock ISSN 09628452.
\newblock \doi{10.1098/rspb.1982.0085}.

\bibitem[Stites et~al.(2021)Stites, Zimmermann, Wu, Sennesh, and Van~de Meent]{Stites2021}
Sam Stites, Heiko Zimmermann, Hao Wu, Eli Sennesh, and Jan-Willem Van~de Meent.
\newblock Learning proposals for probabilistic programs with inference combinators.
\newblock \emph{37th Conference on Uncertainty in Artificial Intelligence (UAI 2021)}, 2021.

\bibitem[Taniguchi et~al.(2022)Taniguchi, Iwasawa, Kumagai, and Matsuo]{Taniguchi2022}
Shohei Taniguchi, Yusuke Iwasawa, Wataru Kumagai, and Yutaka Matsuo.
\newblock Langevin autoencoders for learning deep latent variable models.
\newblock \emph{Advances in Neural Information Processing Systems}, 35:\penalty0 13277--13289, 2022.
\newblock URL \url{https://proceedings.neurips.cc/paper_files/paper/2022/hash/565f995643da6329cec701f26f8579f5-Abstract-Conference.html}.

\bibitem[Webb et~al.(2018)Webb, Golinski, Zinkov, N, Rainforth, Teh, and Wood]{Webb2018}
Stefan Webb, Adam Golinski, Rob Zinkov, Siddharth N, Tom Rainforth, Yee~Whye Teh, and Frank Wood.
\newblock Faithful inversion of generative models for effective amortized inference.
\newblock In S.~Bengio, H.~Wallach, H.~Larochelle, K.~Grauman, N.~Cesa-Bianchi, and R.~Garnett, editors, \emph{Advances in Neural Information Processing Systems}, volume~31. Curran Associates, Inc., 2018.
\newblock URL \url{https://proceedings.neurips.cc/paper_files/paper/2018/file/894b77f805bd94d292574c38c5d628d5-Paper.pdf}.

\bibitem[Welling and Teh(2011)]{Welling2011}
Max Welling and Yee~Whye Teh.
\newblock Bayesian learning via stochastic gradient langevin dynamics.
\newblock In \emph{Proceedings of the 28th International Conference on Machine Learning, ICML 2011}, Bellevue, WA, USA, 2011. Proceedings of Machine Learning Research.

\bibitem[Whittington and Bogacz(2017)]{whittington2017approximation}
James~CR Whittington and Rafal Bogacz.
\newblock An approximation of the error backpropagation algorithm in a predictive coding network with local hebbian synaptic plasticity.
\newblock \emph{Neural computation}, 29\penalty0 (5):\penalty0 1229--1262, 2017.

\bibitem[Xiao et~al.(2017)Xiao, Rasul, and Vollgraf]{xiao2017fashion}
Han Xiao, Kashif Rasul, and Roland Vollgraf.
\newblock Fashion-mnist: a novel image dataset for benchmarking machine learning algorithms.
\newblock \emph{arXiv preprint arXiv:1708.07747}, 2017.

\bibitem[Yin and Ao(2006)]{Yin2006}
L.~Yin and P.~Ao.
\newblock Existence and construction of dynamical potential in nonequilibrium processes without detailed balance.
\newblock \emph{Journal of Physics A: Mathematical and General}, 39\penalty0 (27):\penalty0 8593, June 2006.
\newblock ISSN 0305-4470.
\newblock \doi{10.1088/0305-4470/39/27/003}.
\newblock URL \url{https://dx.doi.org/10.1088/0305-4470/39/27/003}.

\bibitem[Zahid et~al.(2024)Zahid, Guo, and Fountas]{Zahid2024}
Umais Zahid, Qinghai Guo, and Zafeirios Fountas.
\newblock Sample as you infer: Predictive coding with langevin dynamics.
\newblock In \emph{Proceedings of the 41st International Conference on Machine Learning}, volume 235, Vienna, Austria, 2024. Proceedings of Machine Learning Research.

\end{thebibliography}
\bibliographystyle{plainnat}


\appendix

\section{Further experiments and results}
\label{app:experiments}

\paragraph{Alternate image generation/ representation learning}
As indicated in Section~\ref{sec:background}, this paper builds upon the particle gradient descent (PGD) algorithm; \citet{Kuntz2023} demonstrated the algorithm's performance by training a generator network on CelebA. Their network employed a Gaussian likelihood with a fixed standard deviation of $0.01$, and evaluated a log-joint objective over 100 epochs on exactly 10,000 subsampled data points. The paper then evaluated mean squared error on an inpainting task and the Frechet Inception Distance over data images.

When applied to the resulting target density, DCPC amounts to PGD with a resampling step. Table~\ref{tab:celeba_vs_pgd} shows the results of training and evaluating the same model described above with DCPC. Since PGD trained for 100 epochs with a batch size of 128, albeit on a 10,000-image subsample of CelebA, we trained with the entire dataset for 100 epochs with batch-size 128.

\begin{table}[!h]
    \centering
    \begin{tabular}{lllll}
        Inference type & Log-joint & FID $\downarrow$ \\
        \hline
        VAE ($K=10$) & $-4.3 \times 10^{5}$ & $171.5 \pm 0.1$ \\
        PGD ($K=10$) & $-3.8 \times 10^{5}$ & $100 \pm 2.7$\\
        DCPC (ours, $K=10$) & $-1.8 \times 10^{6}$ & $\mathbf{82.7} \pm 0.9$ \\
        \hline
    \end{tabular}
    \caption{Log-joint probabilities and FID metrics show how DCPC performs against the original PGD.}
    \label{tab:celeba_vs_pgd}
    \vspace{-1em}
\end{table}
We suspect that the \href{https://github.com/juankuntz/ParEM/blob/727554d9e06f8311c203d4fb2b66f102d42e8ee0/torch/parem/algorithms.py#L303}{supplied code for log-joint calculation} averages over either particles or batch items differently from how we have evaluated DCPC (e.g.~we call \texttt{mean()} without dividing by any further shape dimensions), accounting for the apparent order-of-magnitude difference between log-joints.

At the request of reviewers, we have substituted a simplified Figure~\ref{fig:dcpc-structure} in the main text for Figure~\ref{fig:laminar_circuitry} showing how to map DCPC onto laminar microcircuit structure.

\begin{figure}
    \centering
    \includegraphics[width=0.8\columnwidth]{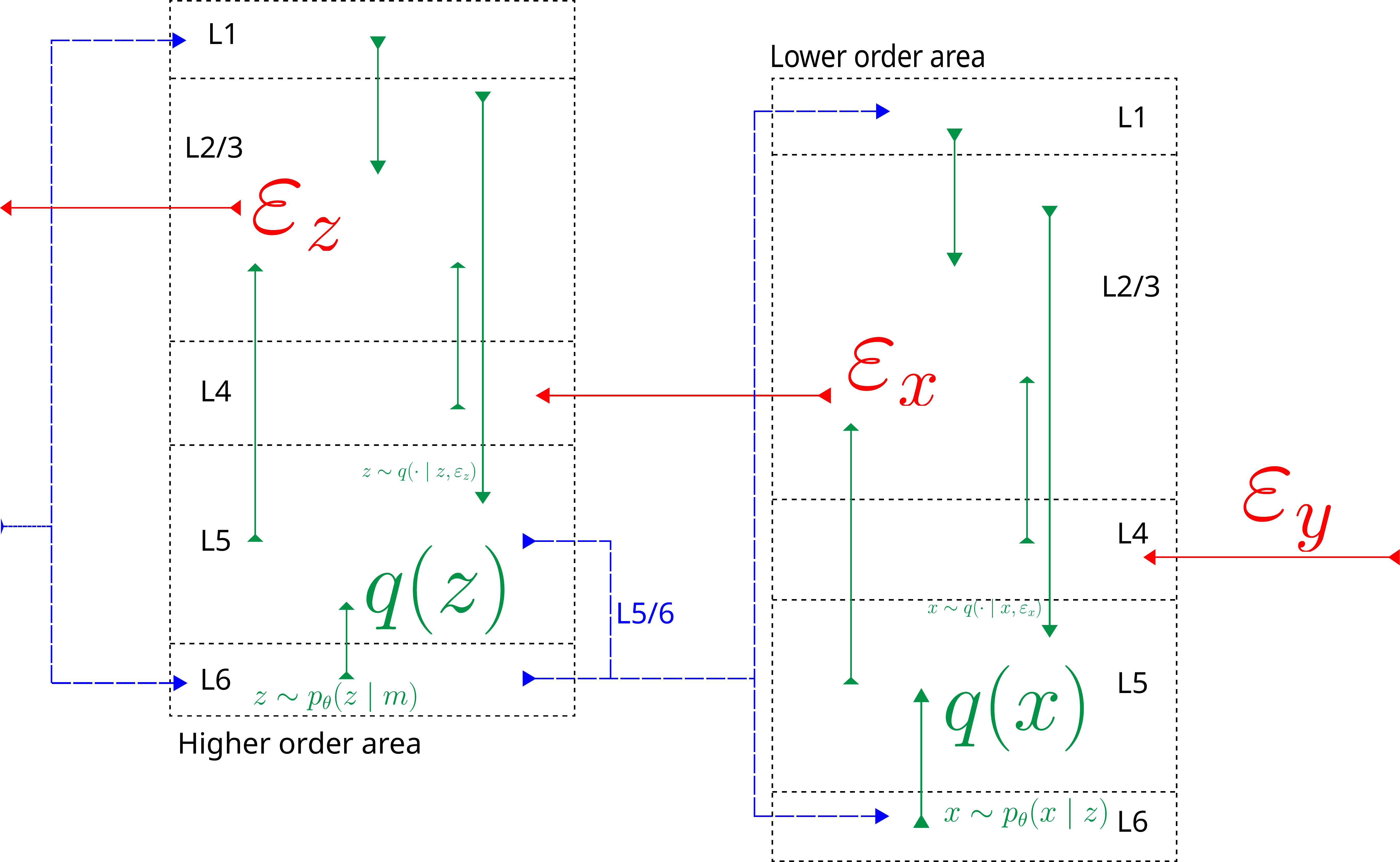}
    \caption{Divide-and-conquer predictive coding provides an algorithmic interpretation for some of the connections mapped in the canonical neocortical microcircuit~\citep{Bastos2012,Bastos2020,Campagnola2022}: prediction errors (red) arrive through ascending pathways into the central laminar layer 4, which transmits them up to layers 2/3 (green). These layers combine the incoming errors with a present posterior estimate (green L5$\rightarrow$ L2/3 connection) to generate prediction errors for the next cortical area. Eventually the updated predictions flow back down the cortical hierarchy (blue).}
    \label{fig:laminar_circuitry}
    \vspace{-1em}
\end{figure}

\section{Importance sampling and gradient estimation proofs}
\label{app:formal_theorems}

\citetapp{Titsias2023} introduced optimal estimators for preconditioning Langevin dynamics to adapt with the Fisher information of the target density. Definition~\ref{def:bayesian_fisher} gives the most basic estimator for that Fisher information, defined in terms of the score functions we use as prediction errors.
\begin{definition}[Bayesian Fisher estimator~\citepapp{Titsias2023}]
\label{def:bayesian_fisher}
Denoting by $\varepsilon_{z}$ the score function (from Equation~\ref{eq:pe}) and letting $\lambda > 0$ be a fixed hyperparameter, the \emph{Bayesian Fisher estimator}
\begin{align}
    \label{eq:bayesian_fisher}
    \Tilde{\mathcal{I}}_{K} &:= \expect{\varepsilon^{1:K}_{z} \sim \posteriorc{\theta}{z}{\mathbf{z}\setminus z}}{
        \left(\varepsilon_{z} \varepsilon^{\top}_{z}\right)^{k}
    }  + \frac{\lambda}{K} I
\end{align}
is an empirical estimator of a Bayesian target density's Fisher information (the target density from which the prediction errors were derived) based on a cloud of $K$ particles. When the particles are not yet distributed around a mode of the target (e.g.~the score function does not have an average of zero), substituting the empirical covariance for the expected outer product reduces the estimator's bias
\begin{align}
    \label{eq:bayesian_empirical_fisher}
    \Hat{\mathcal{I}}_{K}(\varepsilon^{1:K}_{z}) &:= \mathrm{Cov}\left(\varepsilon^{1:K}_{z}\right) + \frac{\lambda}{K} I.
\end{align}
\end{definition}
The above matrix does not describe the preconditioner that \citetapp{Titsias2023} actually recommended applying in a Langevin proposal. Definition~\ref{def:pc_fim} provides the fully normalized preconditioner.
\begin{definition}[Predictive coding Fisher preconditioner~\citepapp{Titsias2023}]
\label{def:pc_fim}
Using Definition~\ref{def:bayesian_fisher} to parameterize a preconditioner for the Langevin dynamics proposal, the \emph{predictive coding Fisher preconditioner} is the inverse of Equation~\ref{eq:bayesian_empirical_fisher}, normalized to have an average eigenvalue of 1
\begin{align}
    \label{eq:pc_fim}
    \Hat{\Sigma}_{\mathcal{I}}(\varepsilon^{1:K}_{z}) &:= \frac{
        \Hat{\mathcal{I}}_{K}(\varepsilon^{1:K}_{z})^{-1}
    }{
        \frac{1}{d} \mathrm{Tr}[\Hat{\mathcal{I}}_{K}(\varepsilon^{1:K}_{z})^{-1}]
    }.
\end{align}
\end{definition}

Definition~\ref{def:spw_density} generalizes the definition of importance sampling, suitable for recursively constructing sequential Monte Carlo algorithms with changing target densities.
\begin{definition}[Strict proper weighting for a density]
\label{def:spw_density}
Given an unnormalized density $\targetm{\theta}{\mathbf{z}}$ with corresponding normalizing constant $\normalizer{}{\theta}$ and normalized density $\posteriorm{\theta}{\mathbf{z}}$
\begin{align*}
    \normalizer{}{\theta} &:= \expectint{d\mathbf{z}}{\mathbf{z} \in \mathcal{Z}}{\targetm{\theta}{\mathbf{z}}} &
    \posteriorm{\theta}{\mathbf{z}} &:= \frac{\targetm{\theta}{\mathbf{z}}}{\normalizer{}{\theta}},
\end{align*}
the random variables $w, \mathbf{z} \sim \proposalm{}{w, \mathbf{z}}$ are \emph{strictly properly weighted}~\citep{Naesseth2015} with respect to $\targetm{\theta}{\mathbf{z}}$ if and only if for any measurable test function $h: \mathcal{Z} \rightarrow \mathbb{R}$, the weighted expectation over the proposal $\proposalm{}{w, \mathbf{z}}$ equals the expectation under the target $\targetm{\theta}{\mathbf{z}}$
\begin{align}
    \label{eq:spw_density}
    \expect{w, \mathbf{z} \sim q(w, \mathbf{z})}{w h(\mathbf{z})} &= \expectint{\targetm{\theta}{\mathbf{z}}\: d\mathbf{z}}{\mathbf{z} \in \mathcal{Z}}{h(\mathbf{z})}.
\end{align}
\end{definition}

The following two propositions come from the previous work by \citetapp{Wu2020}, \citet{Stites2021} and \citetapp{Zimmermann2021}. The reader looking for foundations can see \citet{Naesseth2015} or \citetapp{Chopin2020}.
\begin{proposition}[The free energy upper-bounds the surprisal]
\label{prop:elbo_bound}
Given a proposal $\proposalm{\phi}{w, \mathbf{z}}$ strictly properly weighted (Definition~\ref{def:spw_density}) for the target $\targetm{\theta}{\mathbf{z}}$, the variational free energy provides an upper bound to the target's surprisal
\begin{align}
    \label{eq:vfe_bound}
    \mathcal{F}(\theta, q) &\geq -\log \normalizer{}{\theta}.
\end{align}
\end{proposition}
\begin{proof}
I begin by writing out the free energy (Equation~\ref{eq:vfe}) as an expectation of a negative logarithm
\begin{align*}
    \mathcal{F}(\theta, q) &= \expect{z, w \sim \proposalm{}{z, w}}{-\log w}.
\end{align*}
Jensen's Inequality allows moving the expectation into the negative logarithm by relaxing the definition of the variational free energy from an equality to an upper bound
\begin{align*}
    \mathcal{F}(\theta, q)
    &\geq
    -\log \expect{z, w \sim \proposalm{}{z, w}}{w}.
\end{align*}
Setting $h(z)=1$, strict proper weighting for an unnormalized density (Definition~\ref{def:spw_density}) says the expected weight will be the normalizing constant
\begin{align*}
    \expect{z, w \sim \proposalm{}{z, w}}{w}
    &=
    \normalizer{}{\theta}
\end{align*}
which I substitute back in to obtain the desired inequality $\mathcal{F}(\theta, q) \geq -\log \normalizer{}{\theta}$.
\end{proof}

\begin{proposition}[Weighted expectations approximate the normalized target up to a constant]
\label{prop:spw_expectations}
Given a proposal $\proposalm{\phi}{w, \mathbf{z}}$ strictly properly weighted (Definition~\ref{def:spw_density}) for the target $\targetm{\theta}{\mathbf{z}}$ and a measurable test function $h: \mathcal{Z} \rightarrow \mathbb{R}$, weighted expectations under the proposal equal the target's normalizing constant times the test function's expectation under the normalized target
\begin{align*}
    \expect{(w, \mathbf{z}) \sim \proposalm{\phi}{w, \mathbf{z}}}{w h(\mathbf{z})} &= Z(\theta) \expect{
        \mathbf{z} \sim \posteriorm{\theta}{\cdot}
    }{
        h(\mathbf{z})
    }.
\end{align*}
\end{proposition}
\begin{proof}
Strict proper weighting (Equation~\ref{eq:spw_density}) states that weighted expectations under the proposal equal integrals over the unnormalized target, and by definition the normalized target equals the unnormalized density over its normalizing constant
\begin{align*}
    \expect{w, \mathbf{z} \sim q(w, \mathbf{z})}{w h(\mathbf{z})}
    &=
    \expectint{\targetm{\theta}{\mathbf{z}}\: d\mathbf{z}}{\mathbf{z} \in \mathcal{Z}}{h(\mathbf{z})},
    &
    \posteriorm{\theta}{\mathbf{z}} &:= \frac{\targetm{\theta}{\mathbf{z}}}{\normalizer{}{\theta}}.
\end{align*}
The second equation expresses the unnormalized target in terms of the normalized one
\begin{align*}
    \normalizer{}{\theta} \posteriorm{\theta}{\mathbf{z}} &= \targetm{\theta}{\mathbf{z}},
\end{align*}
and substituting this expression into the definition of strict proper weighting leads to the desired result
\begin{align*}
    \expectint{\targetm{\theta}{\mathbf{z}}\: d\mathbf{z}}{\mathbf{z} \in \mathcal{Z}}{h(\mathbf{z})}
    &=
    \expectint{\normalizer{}{\theta} \posteriorm{\theta}{\mathbf{z}}\: d\mathbf{z}}{\mathbf{z} \in \mathcal{Z}}{h(\mathbf{z})}, \\
    &= \normalizer{}{\theta} \expectint{\posteriorm{\theta}{\mathbf{z}}\: d\mathbf{z}}{\mathbf{z} \in \mathcal{Z}}{h(\mathbf{z})} \\
    \expect{w, \mathbf{z} \sim q(w, \mathbf{z})}{w h(\mathbf{z})} &= \normalizer{}{\theta} \expect{\posteriorm{\theta}{\mathbf{z}}}{h(\mathbf{z})} \qedhere.
\end{align*}
\end{proof}

\begin{proposition}[DCPC's free energy has a pathwise derivative]
\label{prop:free_energy_gradient}
The free energy $\mathcal{F}^{t+1} = \expect{q}{-\log w_{\theta^{t}}^{t+1}}$ constructed by the population predictive coding algorithm (Algorithm~\ref{alg:ppc}) has a pathwise derivative as the expectation of the negative gradient of the log-joint density
\begin{align*}
    \nabla_{\theta^{t}} \mathcal{F}^{t+1} &= \expect{q}{- \nabla_{\theta^{t}} \log \priorm{\theta^{t}}{\mathbf{x}, \mathbf{z}^{t+1}}}.
\end{align*}
\end{proposition}
\begin{proof}
The free energy has an expression in terms of Equation~\ref{eq:target_weight}
\begin{align*}
    \mathcal{F}^{t+1} &= \expect{q}{-\log w_{\theta^{t}}^{t+1}} &
    w_{\theta^{t}}^{t+1} &= \frac{
        \priorm{\theta^{t}}{\mathbf{x}, \mathbf{z}}
    }{
        \prod_{z \in \mathbf{z}} \targetc{\theta}{z^{t+1}_b}{\mathbf{z}_{\setminus z}}
    }
    \prod_{z \in \mathbf{z}} \normalizerHat{\theta^{t}}{\mathbf{z}_{\setminus z}}^{t+1}, \\
    \normalizerHat{\theta^{t}}{\mathbf{z}_{\setminus z}}^{t+1} &= \frac{1}{K} \sum_{k=1}^{K} u^{t+1, k}_{b} &
    u^{t+1}_{z} &= \frac{
        \targetc{\theta}{z^{t+1}}{\mathbf{z}_{\setminus z}}
    }{
        \proposalc{}{z^{t+1}}{\varepsilon_{z}(z^{t})}
    },
\end{align*}
and writing out the free energy itself in full shows that many terms cancel
\begin{align*}
    \proposalc{}{\mathbf{z}^{t+1}}{\mathbf{z}^{t}} &= \prod_{z^{t+1}_b \in \mathbf{z}^{t+1}} \proposalc{}{z^{t+1}}{z^{t}}, \\
    \mathcal{F}^{t+1} &= \expect{\proposalc{}{\mathbf{z}^{t+1}}{\mathbf{z}^{t}}}{-\log \frac{
            \priorm{\theta^{t}}{\mathbf{x}, \mathbf{z}}
        }{
            \prod_{z \in \mathbf{z}} \targetc{\theta}{z^{t+1}_b}{\mathbf{z}_{\setminus z}}
        }
        \prod_{z \in \mathbf{z}} \cancel{\frac{1}{K}} \cancel{\sum_{k=1}^{K}} \frac{
            \targetc{\theta}{z^{t+1}_b}{\mathbf{z}_{\setminus z}}
        }{
            \proposalc{}{z^{t+1}}{\varepsilon_{z}(z^{t})}
        }
    } \\
    &= \expect{\proposalc{}{\mathbf{z}^{t+1}}{\mathbf{z}^{t}}}{-\log \frac{
            \priorm{\theta^{t}}{\mathbf{x}, \mathbf{z}}
        }{
            \cancel{\prod_{z \in \mathbf{z}} \targetc{\theta}{z^{t+1}_b}{\mathbf{z}_{\setminus z}}}
        }
        \frac{
            \cancel{\prod_{z \in \mathbf{z}} \targetc{\theta}{z^{t+1}_b}{\mathbf{z}_{\setminus z}}}
        }{
            \prod_{z \in \mathbf{z}} \proposalc{}{z^{t+1}}{\varepsilon_{z}(z^{t})}
        }
    } \\
    &= \expect{\proposalc{}{\mathbf{z}^{t+1}}{\mathbf{z}^{t}}}{-\log \frac{
            \priorm{\theta^{t}}{\mathbf{x}, \mathbf{z}}
        }{
            \proposalc{}{\mathbf{z}^{t+1}}{\mathbf{z}^{t}}
        }
    }.
\end{align*}
The proposal distribution $q$ is a function of the random variable values themselves through the prediction errors, not of the parameters $\theta$. The above expression therefore admits a pathwise derivative~\citepapp{Schulman2015}, moving the gradient operator into the expectation
\begin{align*}
    \nabla_{\theta^{t}} \mathcal{F}^{t+1} &= \nabla_{\theta^{t}} \expect{\proposalc{}{\mathbf{z}^{t+1}}{\mathbf{z}^{t}}}{-\log \frac{
            \priorm{\theta^{t}}{\mathbf{x}, \mathbf{z}^{t+1}}
        }{
            \proposalc{}{\mathbf{z}^{t+1}}{\mathbf{z}^{t}}
        }
    } \\
    &= \expect{\proposalc{}{\mathbf{z}^{t+1}}{\mathbf{z}^{t}}}{\nabla_{\theta^{t}} -\log \frac{
            \priorm{\theta^{t}}{\mathbf{x}, \mathbf{z}^{t+1}}
        }{
            \proposalc{}{\mathbf{z}^{t+1}}{\mathbf{z}^{t}}
        }
    } \\
    &= \expect{\proposalc{}{\mathbf{z}^{t+1}}{\mathbf{z}^{t}}}{\nabla_{\theta^{t}} - \left[\log \priorm{\theta^{t}}{\mathbf{x}, \mathbf{z}^{t+1}} - \log \proposalc{}{\mathbf{z}^{t+1}}{\mathbf{z}^{t}} \right]
    } \\
    &= \expect{\proposalc{}{\mathbf{z}^{t+1}}{\mathbf{z}^{t}}}{- \left[\nabla_{\theta^{t}} \log \priorm{\theta^{t}}{\mathbf{x}, \mathbf{z}^{t+1}} - \nabla_{\theta^{t}} \log \proposalc{}{\mathbf{z}^{t+1}}{\mathbf{z}^{t}} \right]
    } \\
    \nabla_{\theta^{t}} \mathcal{F}^{t+1} &= \expect{\proposalc{}{\mathbf{z}^{t+1}}{\mathbf{z}^{t}}}{- \nabla_{\theta^{t}} \log \priorm{\theta^{t}}{\mathbf{x}, \mathbf{z}^{t+1}}}. \qedhere
\end{align*}
\end{proof}

\begin{proposition}[DCPC coordinate updates are strictly properly weighted for the complete conditionals]
\label{prop:ppc_complete_conditionals}
Each DCPC coordinate update (Equation~\ref{eq:coord_weight}) for a latent variable $z \in \mathbf{z}$ is strictly properly weighted (Definition~\ref{def:spw_density}) for $z$'s unnormalized complete conditional. For every measurable $h: \mathcal{Z} \rightarrow \mathbb{R}$
\begin{align}
    \label{eq:ppc_complete_conditionals}
    \expect{z \sim \proposalc{\eta}{z^{t}}{z^{t-1}, \varepsilon^{t}_{z}}}{
        \expect{u \sim \delta(u), z', \Hat{Z} \sim \textsc{RESAMPLE}\left(z, u_{z} \right)}{
            h(z)
        }
    } &= \expectint{\targetc{\theta}{z}{\mathbf{z}_{\setminus z}}\: dz}{z \in \mathcal{Z}}{h(z)}.
\end{align}
\end{proposition}
\begin{proof}
Expanding the outer expectation into an integral and replacing the Dirac delta with the expression for the local weights transforms Equation~\ref{eq:ppc_complete_conditionals} into
\begin{multline*}
    \expectint{\cancel{\proposalc{\eta}{z}{z^{t-1}, \varepsilon^{t}_{z}}}\: dz}{z \in \mathcal{Z}}{
        \frac{
            \targetc{\theta}{z}{\mathbf{z}_{\setminus z}}
        }{
            \cancel{\proposalc{\eta}{z}{z^{t-1}, \varepsilon^{t}_{z}}}
        }
        \expect{z' \sim \textsc{RESAMPLE}\left(z, u_{z} \right)}{
            h(z')
        }
    } =\\ \expectint{\targetc{\theta}{z}{\mathbf{z}_{\setminus z}}\: dz}{z \in \mathcal{Z}}{h(z)};
\end{multline*}
importance resampling also preserves strict proper weighting (see \citet{Naesseth2015,Stites2021} and \citetapp{Chopin2020} for proofs), and so this yields
\begin{align*}
    \expectint{\targetc{\theta}{z}{\mathbf{z}_{\setminus z}} \: dz}{z \in \mathcal{Z}}{
        \expect{z' \sim \textsc{RESAMPLE}\left(z, u_{z} \right)}{
            h(z')
        }
    } &= \expectint{\targetc{\theta}{z}{\mathbf{z}_{\setminus z}}\: dz}{z \in \mathcal{Z}}{h(z)} \\
    \expectint{\targetc{\theta}{z'}{\mathbf{z}_{\setminus z}} \: dz'}{z' \in \mathcal{Z}}{
        h(z')
    } &= \expectint{\targetc{\theta}{z}{\mathbf{z}_{\setminus z}}\: dz}{z \in \mathcal{Z}}{h(z)}.
\end{align*}
\end{proof}

\begin{corollary}[DCPC coordinate updates sample from the true complete conditionals]
\label{cor:ppc_complete_conditionals}
Each DCPC coordinate update (Equation~\ref{eq:coord_weight}) for a latent $z \in \mathbf{z}$ samples from $z$'s complete conditional (the normalization of Equation~\ref{eq:complete_conditional}). Formally, for every measurable $h: \mathcal{Z} \rightarrow \mathbb{R}$, resampled expectations with respect to the DCPC coordinate update equal those with respect to the complete conditional
\begin{align*}
    \expect{z \sim \proposalc{\eta}{z}{z^{t-1}, \varepsilon^{t}_{z}}}{
        \expect{u \sim \delta(u), z' \sim \textsc{RESAMPLE}\left(z, u_{z} \right)}{
            h(z')
        }
    } &= \expectint{\posteriorc{\theta}{z}{\mathbf{z}_{\setminus z}}\: dz}{z \in \mathcal{Z}}{h(z)}.
\end{align*}
\end{corollary}
\begin{proof}
Proposition~\ref{prop:ppc_complete_conditionals} in Appendix~\ref{app:formal_theorems} provides a lemma
\begin{align*}
    \expect{z \sim \proposalc{\eta}{z^{t}}{z^{t-1}, \varepsilon^{t}_{z}}}{
        \expect{u \sim \delta(u), z', \Hat{Z} \sim \textsc{RESAMPLE}\left(z, u_{z} \right)}{
            h(z')
        }
    } &= \expectint{\targetc{\theta}{z}{\mathbf{z}_{\setminus z}}\: dz}{z \in \mathcal{Z}}{h(z)},
\end{align*}
which we can apply by observing that resampling sums over self-normalized weights
\begin{multline*}
    \expect{z \sim \proposalc{\eta}{z}{z^{t-1}, \varepsilon^{t}_{z}}}{
        \expect{u \sim \delta(u), z' \sim \textsc{RESAMPLE}\left(z, u_{z} \right)}{
            h(z)
        }
    } = \\ \expect{z \sim \proposalc{\eta}{z}{z^{t-1}, \varepsilon^{t}_{z}}}{
        \expect{u \sim \delta(u)}{
            \expect{z' \sim \frac{u \dirac{z}{\cdot}}{\sum u'}}{h(z')}
        }
    },
\end{multline*}
which is just a weighted sum that by Definition~\ref{def:spw_density} is itself properly weighted
\begin{align*}
    \expect{z \sim \proposalc{\eta}{z}{z^{t-1}, \varepsilon^{t}_{z}}}{
        \expect{u \sim \delta(u)}{
            \expect{z' \sim \frac{u \dirac{z}{\cdot}}{\sum u'}}{h(z')}
        }
    }
    &= \expect{z \sim \proposalc{\eta}{z}{z^{t-1}, \varepsilon^{t}_{z}}}{
        \expect{u \sim \delta(u)}{
            \frac{u}{\sum u}
            h(z)
        }
    }
\end{align*}
\begin{align*}
    &= \expect{z \sim \proposalc{\eta}{z}{z^{t-1}, \varepsilon^{t}_{z}}}{
        \expect{u \sim \delta(u)}{
            \frac{1}{\sum u}
            \expectint{\targetc{\theta}{z}{\mathbf{x}, \mathbf{z}_{\setminus}}\: dz}{z \in \mathcal{Z}}{h(z)}
        }
    } \\
    &= \expect{z \sim \proposalc{\eta}{z}{z^{t-1}, \varepsilon^{t}_{z}}}{
        \expect{u \sim \delta(u)}{
            \frac{1}{\cancel{\normalizerHat{\theta}{\mathbf{x}, \mathbf{z}_{\setminus}}}}
            \cancel{\normalizer{\theta}{\mathbf{x}, \mathbf{z}_{\setminus}}}\expectint{\posteriorc{\theta}{z}{\mathbf{x}, \mathbf{z}_{\setminus}}\: dz}{z \in \mathcal{Z}}{h(z)}
        }
    } \\
    &= \expectint{\posteriorc{\theta}{z}{\mathbf{x}, \mathbf{z}_{\setminus}}\: dz}{z \in \mathcal{Z}}{h(z)}. \qedhere
\end{align*}
\end{proof}

\begin{proposition}[DCPC parameter learning requires only local gradients in a factorized generative model]
\label{prop:ppc_local_likelihood}
Consider a graphical model factorized according to Equation~\ref{eq:factorization}, with the additional assumption that the model parameters $\theta \in \Theta = \prod_{x\in \mathbf{x}} \Theta_{x} \times \prod_{z\in \mathbf{z}} \Theta_{z}$ share that factorization. Then the gradient $\nabla_{\theta} \mathcal{F}(\theta, q)$ of DCPC's free energy similarly factorizes into a sum of local particle averages
\begin{align*}
    \nabla_{\theta} \mathcal{F} &= \expect{q}{- \nabla_{\theta} \log \priorm{\theta}{\mathbf{x}, \mathbf{z}}} \\
    &= \sum_{v \in (\mathbf{x}, \mathbf{z})} \expect{\proposalc{}{v, \parents{v}}{\varepsilon_{v}, \varepsilon_{\parents{v}}}}{-\nabla_{\theta_{v}} \log \priorc{\theta_{v}}{v}{\parents{v}}} \\
    &= -\sum_{v \in (\mathbf{x}, \mathbf{z})} \frac{1}{K} \sum_{k=1}^{K}\nabla_{\theta_{v}} \log \priorc{\theta_{v}}{v^{k}}{\parents{v}^{k}}.
\end{align*}
\end{proposition}
\begin{proof}
Proposition~\ref{prop:free_energy_gradient} provides the lemma that $\nabla_{\theta} \mathcal{F} = \expect{q}{- \nabla_{\theta} \log \priorm{\theta}{\mathbf{x}, \mathbf{z}}}$, and applying the factorization of the generative model demonstrates that
\begin{align*}
    \nabla_{\theta} \mathcal{F} &= \expect{q}{- \nabla_{\theta} \sum_{v \in (\mathbf{x}, \mathbf{z})} \log \priorc{\theta}{v}{\parents{v}}}.
\end{align*}
Since the proposal $q$ does not depend on any $\theta$ and consists of a particle cloud, we can rewrite it as a mixture over the particles (after sampling is performed)
\begin{align*}
    \nabla_{\theta} \mathcal{F} &\approx \frac{1}{K} \sum_{k=1}^{K} -\nabla_{\theta} \sum_{v \in (\mathbf{x}, \mathbf{z})} \log \priorc{\theta}{v^{k}}{\parents{v}^{k}},
\end{align*}
and then finally apply the assumption of this theorem that $\theta \in \Theta = \prod_{x\in \mathbf{x}} \Theta_{x} \times \prod_{z\in \mathbf{z}} \Theta_{z}$, moving the gradient operation into the sum over individual random variables
\begin{align*}
    &\approx \frac{1}{K} \sum_{k=1}^{K} \sum_{v \in (\mathbf{x}, \mathbf{z})} -\nabla_{\theta^{v}} \log \priorc{\theta^{v}}{v^{k}}{\parents{v}^{k}}. \qedhere
\end{align*}
\end{proof}

\section{Extension to discrete sample spaces}
\label{app:discrete_wasserstein}
Contemporaneously to the work of \citet{Kuntz2023} on particle gradient descent, \citetapp{Sun2023} derived a novel Wasserstein gradient flow and corresponding descent algorithm for discrete distributions. In their setting, each Wasserstein gradient step constructs a $D$-dimensional, finitely supported distribution over the $C$-Hamming ball of the starting sample, such that the distribution has $DC$ possible states in total. Let $z^{t+h} \in N_{C}(z^{t})$ denote the resulting discrete random variable in the $C$-neighborhood around $z^{t}$ with respect to the Hamming distance. The update rule relies on simulating the gradient flow for time $h$, sampling from a Markov jump process at time $t+h$
\begin{align*}
    z^{t+h} &\sim \prod_{d \in [1\ldots D]} \proposalc{}{z^{t+h}_{d}}{z^{t}_{d}}.
\end{align*}
A rate matrix $Q_d(z^{t})$ defined by the entire discrete variable $z^{t}$ parameterizes the proposal distribution  
\begin{align}
    \label{eq:dlmc_proposal}
    \proposalc{h}{z^{t+h}_d}{z^{t}} &= \exp{\left( Q_d(z^{t}) h \right)}.
\end{align}
the rate matrix will have nondiagonal entries at indices $i \neq j \in [1\ldots C]$ in the neighborhood $N_{C}(z^{t})$, 
\begin{align*}
    Q_d(z^{t})_{i, j} &= w_{i, j} g\left( \frac{
        \posteriorm{\theta}{z^{t}_{\setminus d}, z'_{d, j}}
    }{
        \posteriorm{\theta}{z^{t}_{\setminus d}, z'_{d, i}}
    } \right).
\end{align*}
The above equation requires that $\forall i, j \in [1 \ldots C], w_{i, j} = w_{j, i} \in \mathbb{R}$ and $g(a) = a g\left(\frac{1}{a}\right)$. The ratio of normalized target densities $\pi$ will equal the ratio of unnormalized densities $\gamma$
\begin{align*}
    \frac{
        \posteriorm{\theta}{z^{t}_{\setminus d}, z'_{d,j}}
    }{
        \posteriorm{\theta}{z^{t}_{\setminus d}, z'_{d,i}}
    }
    &= \frac{
        \targetc{\theta}{z'_{d,j}}{z^{t}_{\setminus d}} \cancel{\normalizer{z_d}{z^{t}_{\setminus d}, \theta}}
    }{
        \cancel{\normalizer{z_d}{z^{t}_{\setminus d}}} \targetc{\theta}{z'_{d,i}}{z^{t}_{\setminus d}}
    } \\
    g\left( \frac{
        \posteriorm{\theta}{z^{t}_{\setminus d}, z'_{d,j}}
    }{
        \posteriorm{\theta}{z^{t}_{\setminus d}, z'_{d,i}}
    } \right)
    &= g\left( \frac{
        \targetc{\theta}{z'_{d,j}}{z^{t}_{\setminus d}}
    }{
        \targetc{\theta}{z'_{d,i}}{z^{t}_{\setminus d}}
    } \right).
\end{align*}
Based on the experimental recommendations of \citetapp{Sun2023}, let $w_{i, j} = w_{j, i} = 1$ and $g(a) = \sqrt{a}$. The rate matrix then simplifies to nondiagonal and diagonal terms
\begin{align}
    \label{eq:dlmc_rates}
    Q_d(z^{t})_{i, j} &= \sqrt{
        \frac{
            \targetc{\theta}{z'_{d,j}}{z^{t}_{\setminus d}}
        }{
            \targetc{\theta}{z'_{d,i}}{z^{t}_{\setminus d}}
        }
    }, &
    Q_d(z^{t})_{i, i} &= - \sum_{j \neq i} Q_d(z^{t})_{i, j}.
\end{align}
Equations~\ref{eq:dlmc_proposal} and \ref{eq:dlmc_rates} give a distribution descending the Wasserstein gradient of the free energy with respect to a particle cloud in a discrete sample space. Applying Equation~\ref{eq:dlmc_rates} to $\targetc{\theta}{z}{\mathbf{z}_{\setminus z}}$ yields a factorization in log space
\begin{align*}
    Q(z^{t})_{i, j} &= \sqrt{\frac{
        \targetc{\theta}{z^{t} + i}{\mathbf{z}^{t}_{\setminus z}}
    }{
        \targetc{\theta}{z^{t} + j}{\mathbf{z}^{t}_{\setminus z}}
    }} &
    \log Q(z^{t})_{i, j} &= \frac{1}{2} \left( \log \targetc{\theta}{z^{t} + i}{\mathbf{z}^{t}_{\setminus z}} - \log \targetc{\theta}{z^{t} + j}{\mathbf{z}^{t}_{\setminus z}} \right).
\end{align*}
This difference can be written as a difference of differences
\begin{multline}
    \label{eq:logcc_difference}
    \log \targetc{\theta}{z^{t} + i}{\mathbf{z}^{t}_{\setminus z}} - \log \targetc{\theta}{z^{t} + j}{\mathbf{z}^{t}_{\setminus z}} = \\ \left(\log \targetc{\theta}{z^{t} + i}{\mathbf{z}^{t}_{\setminus z}} - \log \targetc{\theta}{z^{t}}{\mathbf{z}^{t}_{\setminus z}}\right) - \left(\log \targetc{\theta}{z^{t} + j}{\mathbf{z}^{t}_{\setminus z}} - \log \targetc{\theta}{z^{t}}{\mathbf{z}^{t}_{\setminus z}} \right).
\end{multline}

Recent work on efficient sampling for discrete distributions has focused on approximating density ratios, such as the one in Equation~\ref{eq:dlmc_rates}, with series expansions parameterized by error vectors. When the underlying discrete densities consist of exponentiating a differentiable energy function, as in \citetapp{Grathwohl2021}, these error vectors have taken the form of gradients and the finite-series expansions have been Taylor series. When they do not, \citetapp{Xiang2023} showed how they take the form of finite differences and Newton's series
\begin{align}
    \label{eq:newton_series_expansion}
    \log \targetm{}{z'} - \log \targetm{}{z} &\approx \Delta_{1} \left(\log \targetm{}{z}\right)^{\top} \cdot (z' - z).
\end{align}

Discrete DCPC would therefore use finite differences as discrete prediction errors, breaking each discrete $z \in \mathbf{z}$ into dimensions and incrementing each dimension separately to construct a vector
\begin{align}
    \label{eq:finite_difference}
    \Delta_{1} f(z) &:= \left( f(z_{1} + 1, z_{2:D}), \ldots, f(z_{1:i}, z_{i} + 1, z_{i+1:D}), \ldots, f(z_{1:D-1}, z_{D} + 1) \right) \ominus f(z),
\end{align}
where $\ominus$ subtracts the scalar $f(z)$ from the vector elements and $f: \mathbb{Z}^{D} \rightarrow \mathbb{R}$ is the target function. This would lead to defining the discrete prediction error as the finite difference
\begin{align}
    \label{eq:pe_discrete}
    \varepsilon_{z} &:= \Delta_{1}\log \targetc{\theta}{z^{t}}{\mathbf{z}^{t}_{\setminus z}}.
\end{align}

Applying Equation~\ref{eq:newton_series_expansion} to the two terms of Equation~\ref{eq:logcc_difference}, we obtain the approximations
\begin{align*}
    \log \targetc{\theta}{z^{t} + i}{\mathbf{z}^{t}_{\setminus z}} - \log \targetc{\theta}{z^{t}}{\mathbf{z}^{t}_{\setminus z}}
    &\approx
    \Delta_{1} \left(\log \targetc{\theta}{z^{t}}{\mathbf{z}^{t}_{\setminus z}} \right)^{\top} \cdot ((z^{t} + i) - z^{t}) \\
    &\approx \varepsilon_{z}(z^{t})^{\top} \cdot i \\
    \log \targetc{\theta}{z^{t} + j}{\mathbf{z}^{t}_{\setminus z}} - \log \targetc{\theta}{z^{t}}{\mathbf{z}^{t}_{\setminus z}}
    &\approx \Delta_{1} \left(\log \targetc{\theta}{z^{t}}{\mathbf{z}^{t}_{\setminus z}} \right)^{\top} \cdot ((z^{t} + j) - z^{t}) \\
    &\approx \varepsilon_{z}(z^{t})^{\top} \cdot j, \\
    \log Q(z^{t})_{i, j} &\approx \frac{1}{2} \varepsilon_{z}(z^{t})^{\top} \left( i - j \right).
\end{align*}

Discrete DCPC would thus parameterize its discrete proposal (Equation~\ref{eq:dlmc_proposal}) in terms of  $\varepsilon_{z}$ (Equation~\ref{eq:pe_discrete}), so that Equation~\ref{eq:dlmc_rates} comes out to the (matrix) exponential of the (elementwise) exponential
\begin{align*}
    \proposalc{h}{z^{t+h}}{\varepsilon_{z}} &= \exp{\left( Q(\varepsilon_{z}) h \right)} &
    Q_{d}(\varepsilon_{z})_{i, j} &= \exp{\left(\frac{
        (\varepsilon_{z})^{\top}_{d}(i_{d} - j_{d})
    }{2}\right)}.
\end{align*}

\bibliographyapp{neurips_2024}
\bibliographystyleapp{plainnat}

\end{document}